\documentclass[twocolumn]{svjour3}
\usepackage{amsmath,epsfig}
\usepackage{amssymb,amsfonts,amscd,mathrsfs}
\usepackage{graphicx}
\usepackage{euscript,color}
\usepackage{multirow}

\begin{document}

\title{Rate-Distortion Analysis of Multiview Coding in a DIBR Framework}

%

\author{Boshra Rajaei \and
        Thomas Maugey \and
        Hamid-Reza Pourreza \and
        Pascal Frossard}

\institute{B. Rajaei \at
              Ferdowsi University of Mashhad, Iran \\
              Sadjad Institute of Higher Education, Mashhad, Iran\\
              Tel.: +98-511-6029000\\
              \email{boshra.rajaei@stu-mail.um.ac.ir}
           \and
           T. Maugey \at
              Signal Processing Laboratory (LTS4), \'{E}cole Polytechnique F\'{e}d\'{e}rale de Lausanne (EPFL), Switzerland \\
              \email{thomas.maugey@epfl.ch}
           \and
           H.-R. Pourreza \at
              Ferdowsi University of Mashhad, Iran \\
              \email{hpourreza@um.ac.ir}
           \and
           P. Frossard \at
              Signal Processing Laboratory (LTS4), \'{E}cole Polytechnique F\'{e}d\'{e}rale de Lausanne (EPFL), Switzerland \\
              \email{pascal.frossard@epfl.ch}
}

\maketitle
\begin{abstract}
Depth image based rendering techniques for multiview applications have been recently introduced for efficient view generation at arbitrary camera positions. Encoding rate control has thus to consider both texture and depth data. Due to different structures of depth and texture images and their different roles on the rendered views, distributing the available bit budget between them however requires a careful analysis. Information loss due to texture coding affects the value of pixels in synthesized views while errors in depth information lead to shift in objects or unexpected patterns at their boundaries. In this paper, we address the problem of efficient bit allocation between textures and depth data of multiview video sequences. We adopt a rate-distortion framework based on a simplified model of depth and texture images. Our model preserves the main features of depth and texture images. Unlike most recent solutions, our method permits to avoid rendering at encoding time for distortion estimation so that the encoding complexity is not augmented. In addition to this, our model is independent of the underlying inpainting method that is used at decoder. Experiments confirm our theoretical results and the efficiency of our rate allocation strategy.
\end{abstract}
\keywords{depth image based rendering \and multiview video coding \and rate allocation \and rate-distortion analysis}

\section{Introduction}\label{sec:introduction}

Three-dimensional video coding is a research field that has witnessed many technological revolutions in the recent years. One of them is the significant improvement in the capabilities of camera sensors. Nowadays, high quality camera sensors that capture color and depth information are easily accessible \cite{Zhang12}. Obviously this brings important modifications in the data that the 3D transmission systems have to process. A few years ago, transmission systems used disparity to improve the compression performance \cite{Merkle07,Vetro11}. Nowadays, 3D systems rather employ depth information to improve the quality experience by, for example, increasing the number of views that could be displayed at the receiver side \cite{Muller11,Tian09}. This is possible because of  depth image based rendering (DIBR) techniques \cite{Fehn04,Shao11} that project one reference image onto virtual views using depth as geometrical information. Figure \ref{fig:diagram} shows the overall structure of a DIBR multiview coder that is also considered in this paper. It includes the following steps: first, the captured views in addition to their corresponding depth maps are coded at bit rates assigned by a rate allocation method. Then the coded information are transmitted to the decoder. Finally, at the decoder the reference views are decoded and virtual views are synthesized using the depth information. View synthesis consists of two parts; projection into the virtual view location using closest reference views and inpainting for filling the holes \cite{Oh09,Cheng08} or pixels that remain undetermined after projection.

DIBR techniques offer new possibilities but also impose new challenges. One of the important questions relies in the effect of depth compression on the view synthesis performance \cite{Merkle09}; in particular, for a given bit budget $R$, what is the best allocation between depth and texture data or in other words, how can we distribute the total bitrate between color and geometrical information in order to maximize the rendering quality? It is important to note that the quality of the rendered view is of interest here, and not the distortion of depth images \cite{Merkle09,Maitre10}. This renders the problem of rate allocation particularly challenging.

The rate allocation problem has been the topic of many researches in the past few years. Allocating a fixed percentage of total budget to the texture and depth data is probably the simplest allocation policy in the DIBR coding methods \cite{Sanchez09,Daribo08,Milani11}. More efficient methods have however been proposed recently, and we discuss them in more details below.

Starting from the current multiview coding (MVC) profile of H.264/AVC \cite{AVC1,AVC2,Vetro11}, we should mention that MVC uses the distortion of depth maps to distribute the available bit budget between texture and depth images. A group of papers try to improve MVC by taking into account depth properties. In \cite{Ekmekcioglu11}, authors suggest a preprocessing step based on an adaptive local median filter to enhance spatial, temporal and inter-view correlations between depth maps and consequently, improve the performance of MVC. Using the correlation between reference views, the work in \cite{Lee11} skips some depth blocks in the coding and hence, reduces the required bit budgets for coding depth maps. Other methods try to estimate at encoder the distortion of virtual views, which then replaces the depth map distortion in the mode decision step of the MVC method \cite{AVC1}. In \cite{Liu09}, the authors provide an upper bound for virtual view distortion that is related to the depth and texture errors and the gradients of the original reference views. Another upper bound for rendered view distortion proposed when encoder has access to the original intermediate views at the encoder \cite{Nguyen09}. In \cite{Kim10}, the algorithm calculates the translation error induced by depth coding and then estimates the rendered view distortion from the texture data. In a similar approach, the work in \cite{Oh11} models the distortion at each pixel of a virtual view, including the pixels in occluded regions. These methods only try to improve the current MVC profile and without modeling the distortion rate behavior, they can not be used as general solutions for the rate allocation problem.

Beside improving the current MVC allocation policy, other papers build a complete rate-distortion model to solve the rate allocation problem of distributing total bit budget between texture and depth data in a DIBR multiview coder \cite{Davidoiu11,Maitre08,Wang12,Cheung11,Gelman12}. For example, assuming independency between depth and texture errors, the work in \cite{Davidoiu11} proposes a DR function to find the optimal allocation in a video system with one reference and one virtual view. A region-based approach for estimating the distortion at virtual views is proposed in \cite{Wang12}. Here, the allocation scheme is an iterative algorithm that needs to render one virtual view at every iteration for parameter initialization. This is very costly in terms of computational complexity. Along the same line of research, we also notice the rate allocation and view selection method proposed in \cite{Cheung11}. In this work, the authors first provide a cubic distortion model for synthetic views; they estimate the model coefficients by rendering at least one intermediate view between each reference camera views. Then, using this distortion model, a DR function is formulated and a modified search algorithm is executed to simplify rate allocation. Finally, a DR function is provided for a layer-based depth coder in \cite{Gelman12}. The main drawbacks in the above allocation schemes reside in the rendering of at least one virtual view at encoding time and in the construction of DR functions that are view dependent. Rendering at encoder side dramatically increases the computational complexity of the coder and is therefore not acceptable for realtime applications. In addition for view rendering at arbitrary camera positions, multiview systems require rate allocation strategies that work independently of reference and virtual view numbers and exact positions.

\begin{figure*}[tb]
\centering
\centerline{\epsfig{figure=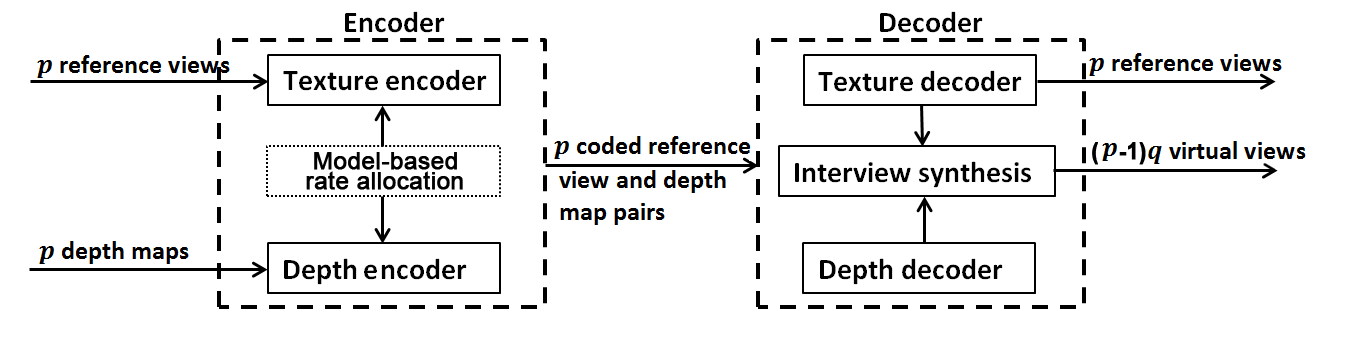,width=15cm}}
\caption{A DIBR multiview coder structure with $p$ reference cameras and $q$ equally spaced virtual views between each two reference views.}
\label{fig:diagram}
\end{figure*}

In this paper, we propose a novel DR model to solve the rate allocation problem in DIBR coding with arbitrary number of reference and virtual views and without rendering at the encoder side. Inspired by \cite{Don99,Pen05,Mal12}, we first simplify different aspects of a multiview coder and keeping only the main features. In particular, we make simple models for depth and texture coders, camera setup and scene under observation. Then, using a rate-distortion framework, a DR function is calculated and eventually is used for optimizing the allocation problem in multiview coding. An important property of our allocation method is that, we do not consider the inpainting step for virtual view synthesis at the decoder. There are two reasons for this choice: first, we want to design an allocation strategy that is independent of the actual inpainting method; second, we focus on the effect of view projections, which is mostly related to the geometry of the scene. Experimental results show that our model-based rate allocation method is efficient for different system configurations. The approach proposed in this paper has low complexity but provides a distortion that is not far from optimum, and in particular it outperforms a priori rate allocation strategies that are commonly used in practice.

The organization of this paper is as follows. Next section clarifies some notations, camera and scene model and rate-distortion framework as it is used in Section \ref{sec:proofs} for calculation of our allocation model. Section \ref{sec:implementation} addresses a few optimization issues. Finally, Section \ref{sec:simulations} includes the details of our experimental results, parameter values and comparison to other allocation strategies.

\section{Framework and model}\label{sec:framework}
In this section we define a few preliminary concepts that are used in our rate distortion study. Our main focus is the problem of distributing the available bit rate between several reference views and depth maps in a DIBR multiview video application, such that the distortion over all reference and rendered views at the decoder is minimized. In particular we are interested in constructing a rate-distortion for rate allocation without explicit view synthesis at the encoder.

We first construct a rate-distortion model for a typical wavelet-based texture coder and a simple quantized-based depth map coder, along with a simple model of scene. We present below some general notations and the wavelet framework. Then we describe our rate-distortion analysis framework, our model of the scene and of the camera.

\subsection{Notation}\label{ssec:Notation}
Let $\mathbb{R}$ be the set of real numbers. The $L_2$-norm of a function $f: \mathbb{R}^2  \rightarrow \mathbb{R}$ is defined as
$\|f\|_2 = \left( \iint f^2(t_1,t_2) dt_1dt_2 \right)^\frac{1}{2} $. Then, $L_2(\mathbb{R}^2)$ is the set of all functions $f: \mathbb{R}^2  \rightarrow \mathbb{R}$ with a finite $L_2$-norm. The angle bracket represents the inner product of two functions in this space, i.e., for $f, g \in L_2(\mathbb{R}^2)$ we have
\[
\langle f,g  \rangle =  \iint f(t_1,t_2)g(t_1,t_2) dt_1dt_2.
\]

Then, let $\phi: \mathbb{R} \rightarrow \mathbb{R}$ and $\psi: \mathbb{R} \rightarrow \mathbb{R}$ be the univariate scaling and wavelet functions of an orthonormal wavelet transform, respectively \cite{Mal97}. The shifted and scaled forms of these functions are denoted by $\psi_{s,n}(t) = 2^{s/2} \psi(2^st-n)$ and $\phi_{s,n}(t) = 2^{s/2} \phi(2^st-n)$, where $s,n \in \mathbb{Z}$ are respectively the scaling and shifting parameters and $\mathbb{Z}$ is the set of integer numbers. The most standard construction of two-dimensional wavelets relies on a separable design that uses $\Psi_{s,n_1,n_2}^1(t_1,t_2) = \phi_{s,n_1}(t_1) \psi_{s,n_2}(t_2)$, $\Psi_{s,n_1,n_2}^2(t_1,t_2) = \psi_{s,n_1}(t_1) \phi_{s,n_2}(t_2)$, and $\Psi_{s,n_1,n_2}^3(t_1,t_2) = \psi_{s,n_1}(t_1) \psi_{s,n_2}(t_2)$ as the bases. It is proved in \cite{Mal97} that separable wavelets provide an orthonormal basis for $L_2(\mathbb{R}^2)$. Therefore, any function $f \in  L_2(\mathbb{R}^2)$ can be written as
\[
f(t_1,t_2) = \sum_{s, n_1,n_2} \sum_{i=1}^3 C_{s,n_1,n_2}^i \Psi_{s,n_1,n_2}^i(t_1,t_2),
\]
where, for every $s,n_1, n_2 \in \mathbb{Z}$,
\[
C_{s,n_1,n_2}^i=\langle f,\Psi_{s,n_1,n_2}^i\rangle, \  i=1,2,3.
\]

Practically, the wavelet transform defines a scale $s_0$ as the largest scale value. If we call $C_{s,n_1,n_2}^i$ high frequency bands, at $s_0$ we thus have only one low frequency band $\langle f,\Phi_{s_0,n_1,n_2}\rangle$, where $\Phi_{s_0,n_1,n_2}(t_1,t_2)\mbox{ }= \phi_{s_0,n_1}(t_1) \phi_{s_0,n_2}(t_2)$.

\subsection{Scene and camera configuration model}\label{ssec:scene_model}
We use a very simple model of the scene in our analysis we consider foreground objects with arbitrary shapes and flat surfaces on a flat background\footnote{The extension of our analysis to the scenes with $C^\alpha$ regular surfaces is straightforward.}. Additionally, even though a real scene is 3D, our model is a collection of 2D images as we consider projections of the 3D scene into cameras 2D coordinates.

Let $\mathcal{H}^Q(\mathbf{\Omega})$ be the space of 2D functions, $f: \mathbb{R}^2 \rightarrow \mathbb{R}$, on the interval $[0,1]^2 \subset \mathbb{R}^2$, where $Q$ is the number of foreground objects and $\mathbf{\Omega}=\{\Omega_i, i=0, \dots, Q-1\}$ denotes the foreground objects. We define $f\in \mathcal{H}^Q(\mathbf{\Omega})$ as
\begin{eqnarray}
f(t_1,t_2) = \left\{\begin{array}{ll}
1, & \quad \mbox{if $\exists i: (t_1, t_2)\in \Omega_i$}\\
0, &  \quad  \mbox{otherwise}
\end{array}\right.
\end{eqnarray}

Our RD analysis is performed on $\mathcal{H}^1(\mathbf{\Omega})$ where $\mathbf{\Omega}=\{\Omega_0\}$. The extension to multiple foreground objects follows naturally. For the sake of clarity, we skip superscript notation and represent this class by $\mathcal{H}(\Omega)$. Figure \ref{fig:sample_scenemodel} shows a sample function from $\mathcal{H}(\Omega)$. This figure shows one arbitrary shape foreground object on a flat background as it is projected into a 2D camera plane.

In addition to our simple scene model, we describe now our camera configuration model. Let us denote as $\mathcal{B}^p_q(\mathbf{\mathcal{P}})$ a configuration with $p$ reference cameras and $q$ equally spaced intermediate views between each two reference views. Then, $\mathbf{\mathcal{P}}$ is the set of intrinsic and extrinsic parameters for reference and virtual cameras. It is defined as $\mathbf{\mathcal{P}}=\{(A_i,R_i,T_i): i=0, \dots, p-1\}\cup \{(A'_j,R'_j,T'_j): j=0, \dots, (p-1)q\}$, where $A_i$ and $R_i$ are respectively the intrinsic and rotation matrices of $i$th reference camera and $T_i$ is its corresponding translation vector. The same parameters for virtual cameras are given by $A'_j$, $R'_j$ and $T'_j$. Figure \ref{fig:diagram} shows a multiview coder that corresponds to a $\mathcal{B}^p_q(\mathbf{\mathcal{P}})$ configuration. In this paper, we consider that a texture image and a depth map are coded and are sent to the decoder for each reference view. In our camera configuration $\mathcal{B}^p_q(\mathbf{\mathcal{P}})$, we have $p$ pairs of texture images and depth maps to be coded at each time slot. The number of coded views is given by system design criteria or rate-distortion constraints \cite{Cheung11}.

\begin{figure}[tb]
\centering
\centerline{\epsfig{figure=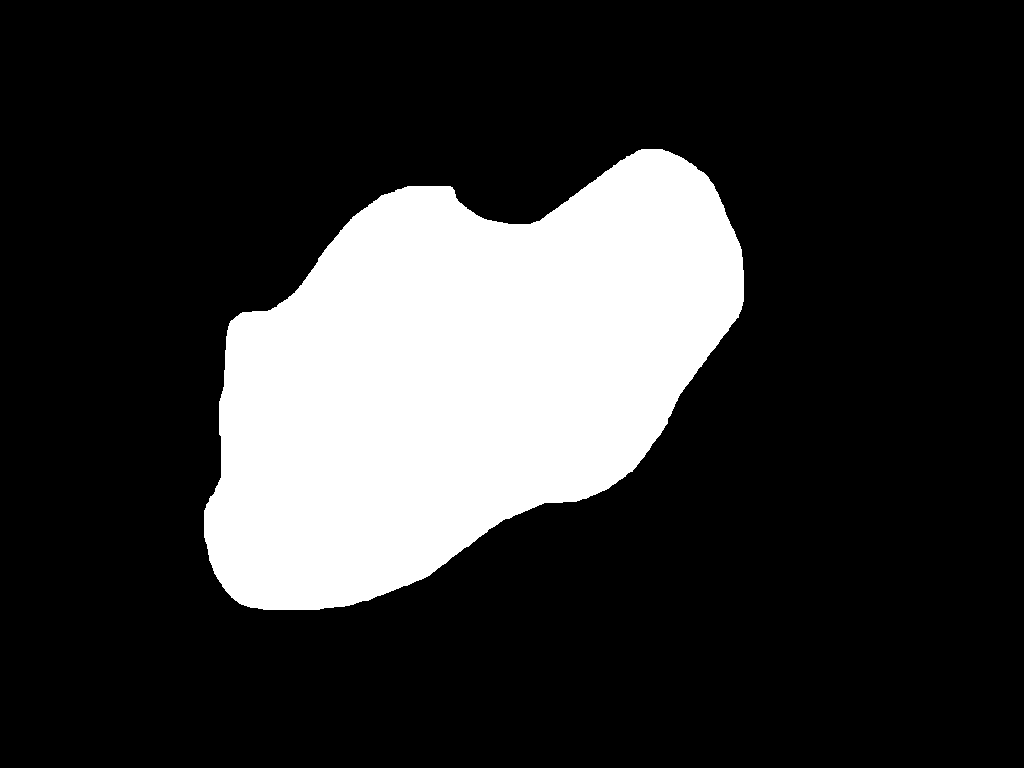,width=5cm}}
\caption{A sample function in $\mathcal{H}^1(\mathbf{\Omega})$.}
\label{fig:sample_scenemodel}
\end{figure}

\subsection{Rate-distortion framework}\label{ssec:RDframework}
Let us define three classes of signals $\mathcal{T} \subset L_2(\mathbb{R}^2)$, $\mathcal{V} \subset L_2(\mathbb{R}^2)$ and $\mathcal{D} \subset L_2(\mathbb{R}^2)$ as reference images, virtual views and depth maps, respectively. Then, define $\mathcal{F}$ as the class of all $f=\{(t_i,d_i):t_i \in \mathcal{T}, d_i \in \mathcal{D}, i=0, \dots, p-1\}$ and similarly, $\mathcal{G}$ as the class of all $g=\{(t_i,v_j):t_i \in \mathcal{T}, v_j \in \mathcal{V}, i=0, \dots, p-1, j=0, \dots, q-1\}$. Here, $\mathcal{F}$ represents all the coded data and $\mathcal{G}$ indicates the set of all reference and virtual views that are reconstructed at the deocer.

A typical multiview video coding strategy consists of at least three building blocks namely, encoder, decoder and rendering algorithm. Consider a texture encoding scheme $\mathcal{E}_\mathcal{T}: \mathcal{T} \rightarrow \{1,2, \ldots, 2^{R_\mathcal{T}}\}$ and similarly a depth encoding scheme $\mathcal{E}_\mathcal{D}: \mathcal{D} \rightarrow \{1,2, \ldots, 2^{R_\mathcal{D}}\}$, where $R_\mathcal{T}=\sum_{i=0}^{p-1}R_{t_i}$ and $R_\mathcal{D}=\sum_{i=0}^{p-1}R_{d_i}$ are the total number of allocated bits to texture and depth information, respectively. This represents a total rate $R=R_\mathcal{T}+R_\mathcal{D}$ bit at the encoder. Correspondingly, we call the texture and depth decoders as $\Gamma_\mathcal{T}: \{1,2, \ldots, 2^{R_\mathcal{T}}\} \rightarrow \mathcal{T}$ and $\Gamma_\mathcal{D}: \{1,2, \ldots, 2^{R_\mathcal{D}}\} \rightarrow \mathcal{D}$. Finally, we denote the rendering scheme as $\Upsilon: \mathcal{F} \rightarrow \mathcal{G}$. Each rendering scheme has two parts: first, the projection into intermediate view using a few closer reference views and their associated depth maps and second, filling the holes that are not covered by any of these reference views. In this paper we are using only the two closest reference views for rendering. Furthermore, we assume in our theoretical analysis that we have no hole in the reconstructed images. Thus, rendering becomes a simple projection of the closer reference views on an intermediate view using depth information. 

Let us denote the decoded data as $\hat{f} = \Gamma_R(\mathcal{E}_R(f))$. The distortion in the rendered version of the data, $\hat{g} = \Upsilon(\hat{f})$, and the original version, $g = \Upsilon(f)$, is given by\footnote{In this paper we consider the $\ell_2$ distortion. However extensions to other norm losses is straightforward.}
\begin{equation}\label{equ:D}
D(g,\hat{g}) = \sum_{i=0}^{p-1}\|t_i-\hat{t}_i\|_2+\sum_{j=0}^{q-1}\|v_j-\hat{v}_j\|_2.
\end{equation}
We finally define the distortion of the coding scheme as the distortion of the encoding algorithm in the least favorable case, i.e.,
\begin{equation}\label{equ:D(R)definition}
D_{\mathcal{E}, \Gamma, \Upsilon}(R) = \sup_{g \in \mathcal{G}} D(g, \hat{g}).
\end{equation}
When the encoding, decoding and rendering strategies are clear from the context we use a simpler notation $D(R)$ and call it the distortion-rate (DR) function.

\section{Theoretical analysis}\label{sec:proofs}
In this section we propose a DR function based on our simple model of scenes $\mathcal{H}^Q(\mathbf{\Omega})$. We first consider a simple camera configuration $\mathcal{B}^1_1(\mathbf{\mathcal{P}})$ with only one reference view and one virtual view. Then we extend analysis to more virtual views with camera configuration $\mathcal{B}^1_q(\mathbf{\mathcal{P}})$ and to more reference views with configuration $\mathcal{B}^p_q(\mathbf{\mathcal{P}})$. For each class of functions the RD analysis is built in the wavelet domain where the distortion is the distance between the original and coded wavelet coefficients. The distortion in wavelet domain is equal to the distortion in the signal domain when wavelets form an orthonormal basis, while such a sparse representation of our virtual and reference views simplifies the RD analysis. Assuming that coding has negligible effect on the average signal value, then we can ignore the distortion in the lowest frequency band. Therefore, in the following analysis we only focus on the distortion in high frequency band coefficients. In all the proofs, we assume that the wavelets have a finite support of length $\ell$ and that their first moments are equal to zero.\\

\begin{theorem}\label{thm:onereference-onevirtual}
The coding scheme that uses wavelet-based texture coder and uniform quantization depth coder, ach-ieves the following DR function on scene configuration $\mathcal{H}^1(\mathbf{\Omega})$ and camera setup $\mathcal{B}^1_1(\mathbf{\mathcal{P}})$
\begin{equation}D(R_t,R_d) \sim O(2\mu \sigma^2 2^{\alpha R_t}+K\frac{\Delta Z}{Z_{min}[2^{\beta R_d}+\Delta Z]}), \nonumber \end{equation}
where $R_t$ and $R_d$ are the texture and depth bit rates, $K=A'R'|T-T'|$ depends on camera parameters, $\Delta Z = Z_{max}-Z_{min}$, $Z_{max}$ and $Z_{min}$ are the maximum and minimum depth values in the scene, $\sigma^2$ is the reference frame variance and $\mu$, $\alpha$ and $\beta$ are positive constants.
\end{theorem}

\begin{proof}
For the camera configuration $\mathcal{B}^1_1(\mathbf{\mathcal{P}})$ we have $g=\{(t_0,v_0)\}$ with one reference view and one virtual view. In all proofs we consider that there is no occluded region for the sake of simplicity. Inspired by \cite{Pra99}, we consider the same quantization level for each wavelet coefficient. This suboptimal choice of quantization will only affect constant factors of the DR function and will not change the final upper bound equation. In addition to this, for depth map coding, we assume a quantization-based coder that simply splits depth image into uniform square areas and for each square the average depth is quantized and coded. Therefore, if we assign $b$ bits for coding each wavelet coefficient in the reference frame and $b'$ bits for coding each depth value, there will be three sources of distortion after decoding and rendering at the decoder side, \\

First at every scale $s$ the number of non-zero wavelet coefficients is $3\times d\Omega \ell2^s$ where $d\Omega$ is the boundary length of $\Omega$ in $v_0$ and 3 factor is because of three wavelet bands. Using the definitions of section \ref{ssec:Notation}, the magnitude of coefficients at scale $s$ of a standard wavelet decomposition is bounded by
\begin{align}
&|C_{s,n_1,n_2}^1| \leq \nonumber\\
&\int_{t_0}^{t_0+\ell2^{-s}}\int_{t_0^\prime}^{t_0^\prime+\ell2^{-s}}|f(t_1,t_2)||\Psi_{s,n_1,n_2}^1(t_1,t_2)|dt_1dt_2 \leq \nonumber\\
& 2^s \int_{t_0}^{t_0+\ell2^{-s}}\int_{t_0^\prime}^{t_0^\prime+\ell2^{-s}}|\phi(2^s t_1-n)\psi(2^s t_2-n)|dt_1dt_2 \leq \nonumber\\
&2^{-s}. \nonumber
\end{align}
We have similar results in case of $|C_{s,n_1,n_2}^2|$ and $|C_{s,n_1,n_2}^3|$. By assigning $b$ bits for coding each coefficient, clearly all the coefficients at scale $s$, $2^{-s}<2^{-b-1}$, will be mapped into zero. Therefore, the first source of coding distortion $D_1$ is
\begin{equation}\label{equ:D_1}
D_1= 3\ell d\Omega \sum_{s=b+2}^\infty 2^s\times(2^{-s})^2=c_1 2^{-b}
\end{equation}
where $c_1=12\ell d\Omega$. Note that a factor of 2 is added here because the error of skipping small wavelet coefficients affects distortion in both $t_0$ and $v_0$ similarly.\\

Then, depth map quantization also introduces distortion as it leads to shifts in foreground objects. Recall that we are calculating distortion in the wavelet domain. Consider $s_1$ as the largest scale with wavelet support length that is smaller than the amount of shift in foreground object. Non-zero wavelet coefficients at scales larger or equal to $s_1$ suffer from position changes due to depth coding. Assume $\Delta_0$ as the maximum position error in $v_0$ with a $b'$ bits quantization-based depth coder. Then we have $\ell2^{-s_1-1}<\Delta_0<\ell2^{-s_1}$. Hence, our second source of error, $D_2$, is
\begin{equation}\label{equ:D_2}
D_2=2\times 3\ell d\Omega \sum_{s=s_1+1}^{b+1} 2^s (2^{-s})^2=c_1 (2^{-s_1}-2^{-b-1}).
\end{equation}
Here the factor 2 is due to shift of significant coefficients.\\

Finally distortion is generated by quantization of non-zero coefficients. Using the definition of $b$ and $s_1$, for the reference frame, $t_0$, we have large coefficients quantization error in $s\leq b+1$ and for virtual view, $v_0$, this happens at $s\leq s_1$. Thus, for the last source of distortion we have
\begin{equation}\label{equ:D_3}
\begin{split}
D_3&=3\ell d\Omega [\sum_{s=1}^{b+1}2^j (2^{-b-1})^2+\sum_{s=1}^{s_1}2^s(2^{-b-1})^2] \\
&= c_1 (2^{-b}+2^{s_1}2^{-2b}).
\end{split}
\end{equation}
\\

Using \eqref{equ:D_1}, \eqref{equ:D_2} and \eqref{equ:D_3} the total distortion is
\[D=c_1 [2^{-b}+2^{-s_1}-2^{-b-1}+2^{-b}+2^{s_1}2^{-2b}].\]
From $s_1$ and $\Delta_0$ definitions we have $s_1\leq b$ and $s_1 \geq \log \Delta_0^{-1}-1$. Therefore, we can simplify the above equation as
\[D=O(2^b+\Delta_0).\]
The first term only depends on texture coding errors and the second term on depth quantization. We replace the texture coding term with a simple distortion model $\mu \sigma^2 2^{-\alpha R}$ \cite{Cover06} where $\mu$ and $\alpha$ are model parameters, $\sigma^2$ is the source variance and $R$ is the target bit rate. Using the formulation of maximum shift error in \cite{Kim10} for the depth distortion term we finally have
\begin{equation}
\begin{split}
D(R_t,R_d) & = O(2\mu \sigma^2 2^{-\alpha R_t}\\
&+A'R'|T-T'|\frac{Z_{max}-Z_{min}}{Z_{min}[2^{\beta R_d}+Z_{max}-Z_{min}]})
\end{split}
\end{equation}
where $\beta$ is another model parameter that depends on depth coding method.\\
\end{proof}

We now extend the above analysis to more complex camera configurations. We first consider $q$ virtual views in a $\mathcal{B}^1_q(\mathbf{\mathcal{P}})$ configuration. \\

\begin{theorem}\label{thm:onereference-multiplevirtual}
The coding scheme that uses wavelet-based texture coder and uniform quantization depth coder, ach-ieves the following DR function on scene configuration $\mathcal{H}^1(\mathbf{\Omega})$ and camera setup $\mathcal{B}^1_q(\mathbf{\mathcal{P}})$
\begin{equation}
\begin{split}
D(R_t,R_d) & \sim O((q+1)\mu \sigma^2 2^{\alpha R_t}\\
& + \sum_{j=0}^{q-1} K_j \frac{\Delta Z}{Z_{min}[2^{\beta R_d}+\Delta Z]}), \nonumber
\end{split}
\end{equation}
where $R_T$ and $R_D$ are the texture and depth rates, $K_j=A'_j R'_j |T-T'_j|$, for $j=0,\dots,q-1$ depends on camera parameters, $\Delta Z = Z_{max}-Z_{min}$, $Z_{max}$ and $Z_{min}$ are the maximum and minimum depth values in the scene, $\sigma^2$ is the reference frame variance and $\mu$, $\alpha$ and $\beta$ are positive constants.\\
\end{theorem}

\begin{proof}
With $q$ virtual cameras the three sources of distortion in the proof of Theorem \ref{thm:onereference-onevirtual} turn into
\begin{equation}
D_1= c_1 (q+1) 2^{-b},
\end{equation}
\begin{equation}
D_2=2\times 3\ell d\Omega \sum_{j=0}^{q-1}\sum_{s=s_j+1}^{b+1} 2^s (2^{-s})^2=c_1 (\sum_{j=0}^{q-1}2^{-s_j}-q2^{-b-1})
\end{equation}
and
\begin{equation}
D_3= c_1 (2^{-b}+2^{-2b}\sum_{j=0}^{q-1}2^{s_j}).
\end{equation}
We have $s_j\leq b$ and $s_j \geq \log \Delta_j^{-1}-1$ for $j=0\dots q-1$, thus
\[D=O((q+1)2^b+\sum_{j=0}^{q-1} \Delta_j).\]

Here, we have simply used the fact that the error in the virtual views augments with the number of such views. The DR function is then obtained by following exactly the same replacements as in the proof of Theorem \ref{thm:onereference-onevirtual}.
\end{proof}

Finally we extend the analysis to configurations with more reference views. We assume that we have equally spaced reference cameras and virtual views, and that the number of intermediate views is uniform between every two reference cameras. A weighted interpolation strategy using the two closest reference views is employed for synthesis at each virtual view point. The weights are related to the distances between corresponding virtual view and right and left reference views similarly to \cite{Liu09}. Theorem \ref{thm:multiplereference-multiplevirtual} provides the general DR function in a general camera configuration with $p$ reference views and $(p-1)q$ virtual views.\\

\begin{theorem}\label{thm:multiplereference-multiplevirtual}
The coding scheme that uses wavelet-based texture coder and uniform quantization depth coder, ach-ieves the following DR function on scene configuration $\mathcal{H}^1(\mathbf{\Omega})$ and camera setup $\mathcal{B}^p_q(\mathbf{\mathcal{P}})$
\begin{equation}
\begin{split}
D(&R_{t_0},\dots, R_{t_{p-1}},R_{d_0},\dots,R_{d_{p-1}}) \sim \\
& O(\sum_{i=0}^{p-1} \mu \sigma_i^2 2^{\alpha R_{t_i}} +\\
& \sum_{j=0}^{(p-1)q} (\frac{d_{j,r}}{d})^2 [\mu \sigma_l^2 2^{\alpha R_{t_l}}+K_{j,l}\frac{\Delta Z}{Z_{min}[2^{\beta R_{d_l}}+\Delta Z]}] +\\
& (\frac{d_{j,l}}{d})^2 [\mu \sigma_r^2 2^{\alpha R_{t_r}}+K_{j,r}\frac{\Delta Z}{Z_{min}[2^{\beta R_{d_r}}+\Delta Z]}]), \nonumber
\end{split}
\end{equation}
where $R_{t_i}$ and $R_{d_i}$ are the texture and depth rates for the $i$th reference view, $\Delta Z = Z_{max}-Z_{min}$, $Z_{max}$ and $Z_{min}$ are the maximum and minimum depth values in the scene, $\sigma_i^2$ is variance of the $i$th reference view and $\mu$, $\alpha$ and $\beta$ are positive constants. Also, $d$ indicates the distance between each two reference cameras and $d_{j,l}$ and $d_{j,r}$ are the distances between $j$th virtual view and its left and right reference cameras. Similarly, we have $K_{j,l}=A'_j R'_j |T_l-T'_j|$ and $K_{j,r}=A'_j R'_j |T_r-T'_j|$ that depend of camera parameters.
\end{theorem}

\begin{proof}
First, using Theorem \ref{thm:onereference-multiplevirtual}, we can write the distortion of a reference view, $r$, and the $q$ virtual views on its left as
\begin{equation}
\begin{split}
D(R_{t_r},R_{d_r}) &= O(\mu \sigma_r^2 2^{\alpha R_{t_r}}\\
&+ \sum_{j=0}^{q-1} [\mu \sigma_r^2 2^{\alpha R_{t_r}} + K_{j,r} \frac{\Delta Z}{Z_{min}[2^{\beta R_{d_r}}+\Delta Z]}])
\end{split}
\end{equation}
Clearly, the first and second terms define the distortion at reference and virtual views, respectively. By adding another reference view, $l$, and using a weighted average of the two closest reference views for synthesizing virtual views we have

\begin{equation}\label{equ:D_thmIII}
\begin{split}
D(R_{t_r}&,R_{t_l},R_{d_r},R_{d_l}) = \\
&O(\mu \sigma_r^2 2^{\alpha R_{t_r}}+\mu \sigma_l^2 2^{\alpha R_{t_l}}+\\
&\sum_{j=0}^{q-1} (\frac{d_{j,r}}{d})^2 [\mu \sigma_l^2 2^{\alpha R_{t_l}}+K_{j,l}\frac{\Delta Z}{Z_{min}[2^{\beta R_{d_l}}+\Delta Z]}] +\\
& (\frac{d_{j,l}}{d})^2 [\mu \sigma_r^2 2^{\alpha R_{t_r}}+K_{j,r}\frac{\Delta Z}{Z_{min}[2^{\beta R_{d_r}}+\Delta Z]}])
\end{split}
\end{equation}
where $d$ indicates the distance between the two reference cameras and $d_{j,l}$ and $d_{j,r}$ are the distances between $j$th virtual view and its left and right reference cameras. Here, our weights are simply related to the distance between virtual view and its neighbor reference views. Finally, summing up the terms of \eqref{equ:D_thmIII} for all reference views, leads to the distortion in Theorem \ref{thm:multiplereference-multiplevirtual}.

\end{proof}

The above rate-distortion analysis is performed on $\mathcal{H}^1(\mathbf{\Omega})$. However, the extension to multiple foreground objects is straightforward and only adds constant factors to the RD function. 

\section{RD Optimization}\label{sec:implementation}
In this section we show how the analysis in Section \ref{sec:proofs} can be used for optimizing the rate allocation in multiview video coding. Using Theorem \ref{thm:multiplereference-multiplevirtual}, the rate allocation problem turns into the following convex nonlinear multivariable optimization problem with linear contraints
\begin{equation}\label{equ: RDoptimization}
\begin{split}
arg \min_{\overrightarrow{R}_t,\overrightarrow{R}_d} & g_t(\overrightarrow{R}_t)+g_d(\overrightarrow{R}_d)\\
\text{such that } & \|\overrightarrow{R}_t+\overrightarrow{R}_d\|_1 \leq R
\end{split}
\end{equation}
where
\[g_t(\overrightarrow{R}_t)=\sum_{i=0}^{p-1} (q+1)\mu \sigma_i^2 2^{\alpha R_{t_i}},\]
\begin{equation}
\begin{split}
g_d(\overrightarrow{R}_d)=\sum_{j=0}^{(p-1)q} &[(\frac{d_{j,l}}{d})^2 K_{j,r}\frac{\Delta Z}{Z_{min}[2^{\beta R_{d_r}}+\Delta Z]} + \\ &(\frac{d_{j,r}}{d})^2 K_{j,l}\frac{\Delta Z}{Z_{min}[2^{\beta R_{d_l}}+\Delta Z]}]\nonumber
\end{split}
\end{equation}
and $R$ is the total target bit rate. The convexity proof is straightforward since the above optimization problem is sum of terms in the form $a2^{-bx}$, which are convex. Therefore it can be solved efficiently using classical convex optimization tools. Note that the above optimization problem is for the general camera configuration $\mathcal{B}^p_q(\mathbf{\mathcal{P}})$. The rate allocation for simpler configurations is straightforward by replacing the objective functions with terms from Theorem \ref{thm:onereference-onevirtual} and \ref{thm:onereference-multiplevirtual}. We can finally note that the rate allocation strategy is only based on encoder side data.

The last issue that we have to address is adjustment of the model parameter. There are three parameters, $\mu$, $\alpha$ and $\beta$ in \eqref{equ: RDoptimization} that we estimate using the following offline method. Using the first texture and depth images, we estimate the model parameters by solving the following regression
\begin{equation}\label{equ:parameter_optimization}
\begin{split}
\min_{\mu,\alpha,\beta} \sum_{k=0}^{n-1} |D(R_k)-D^*(R_k)|
\end{split}
\end{equation}
where $n$ is the number of points in the regression and is further discussed in the next section and $D(R_k)$ is the distortion obtained by our rate allocation strategy of Eq. \eqref{equ: RDoptimization} with target bit rate $R_k$ and $D^*(R_k)$ is the best possible allocation obtained by a full search method at the same bit rate.

\section{Experimental Results} \label{sec:simulations}
In the previous sections, we have studied the bit allocation problem on simple scenes and extracted a model for estimating RD function of a DIBR multiview coder with wavelet-based texture coding and a quantization-based depth coding. This section studies the RD behavior and the accuracy of proposed model on real scenes where JPEG2000 is used for coding depth and reference images.

We use the \emph{Ballet} and \emph{Breakdancers} datasets from Interactive Visual Group of Microsoft Research \cite{microsoft}. In our simulations gray-scale versions of these datasets are used. These datasets contain 100 frames and all the numerical results in this section are the average on the three frames from beginning, middle and end of these sequences, i.e., frames with temporal indices 0, 49 and 99. The camera intrinsic and extrinsic parameters, $\mathcal{P}$, and the scene parameters, $Z_{min}$ and $Z_{max}$, are set to the values given by datasets. In cases where parameters are changed to study the model under some special aspects, we mention the parameter values explicitly.

In an offline stage using Eq. \eqref{equ:parameter_optimization} we adjust $\mu$, $\alpha$ and $\beta$ parameters in Eq. \eqref{equ: RDoptimization} at four bit rates, i.e., $n=4$, for each dataset. The parameter values are set to $(0.9,20.5,8.5)$ for \emph{Ballet} and $(0.9,30.0,2.7)$ for \emph{Breakdancers}. These values are fixed all over this section for the different camera configurations.

In the following sections we study the RD model of Eq. \eqref{equ: RDoptimization} for rate allocation in different camera configurations, $\mathcal{B}^1_1(\mathbf{\mathcal{P}})$, $\mathcal{B}^1_6(\mathbf{\mathcal{P}})$ and $\mathcal{B}^2_3(\mathbf{\mathcal{P}})$. As a comparison criterion we use the optimal allocation that is obtained by rendering all the intermediate views and searching the whole distortion-rate space for the allocation with minimal distortion.

As we want to keep our model independent of any special strategy for filling occluded regions, all occluded regions are ignored in distortion and PSNR calculations.

\subsection{$\mathcal{B}^1_1(\mathbf{\mathcal{P}})$ configuration}\label{ssec:OneRef-OneVir}
We start with $\mathcal{B}^1_1(\mathbf{\mathcal{P}})$ camera setup, a simple configuration with one reference view and only one virtual view. As reference and target cameras, we use the cameras 0 and 1 of the datasets, respectively. Thus, all camera-related parameters in Eq. \eqref{equ: RDoptimization} are set accordingly.

A DR surface is first generated offline for the desire bit rate range to generate the distortion benchmark values. In our study, $R_t$ and $R_d$ are set between 0.02 and 0.5 bpp with 0.02 bpp steps. It means that $R_t$ and $R_d$ axes are discretized into 25 values. Since the images are gray and we are coding only one reference view and one depth map, this range of bit rate is pretty reasonable. The DR surface is generated by actually coding the texture and depth images at each $(R_t,R_d)$ pair and by calculating distortion after decoding and synthesis.

Then, for each target bit rate, $R$, the optimal rate allocation is calculated by cutting the above surface with a plane $R_t+R_d=R$ and minimizing the distortion. If the minimum point occurs between grid points (because we have a discretized surface) bicubic interpolation is used to estimate the optimal allocation. Here, $R$ is set between 0.1 to 0.5 bpp with 0.01 bpp step. Figure \ref{fig:PSNR-OneRef-OneVir} provides distortion curves of compression performance of DIBR coder in terms of PSNR for \emph{Ballet} and \emph{Breakdancers} datasets. The estimated curve is generated by solving the optimization problem provided in Eq. \eqref{equ: RDoptimization} with the proposed RD model. The final PSNR results are averaged over frames 0, 49 and 99 of these datasets. The average differences between the model-based and optimal curves are 0.05 dB and 0.06 dB for \emph{Ballet} and \emph{Breakdancers} sets, respectively. Also, the maximum loss in PSNR in our model-based rate allocation is 0.11 and 0.13 dB, respectively.

\begin{figure*}[tb]

\begin{minipage}[b]{.49\linewidth}
\centering
\rightline{\epsfig{figure=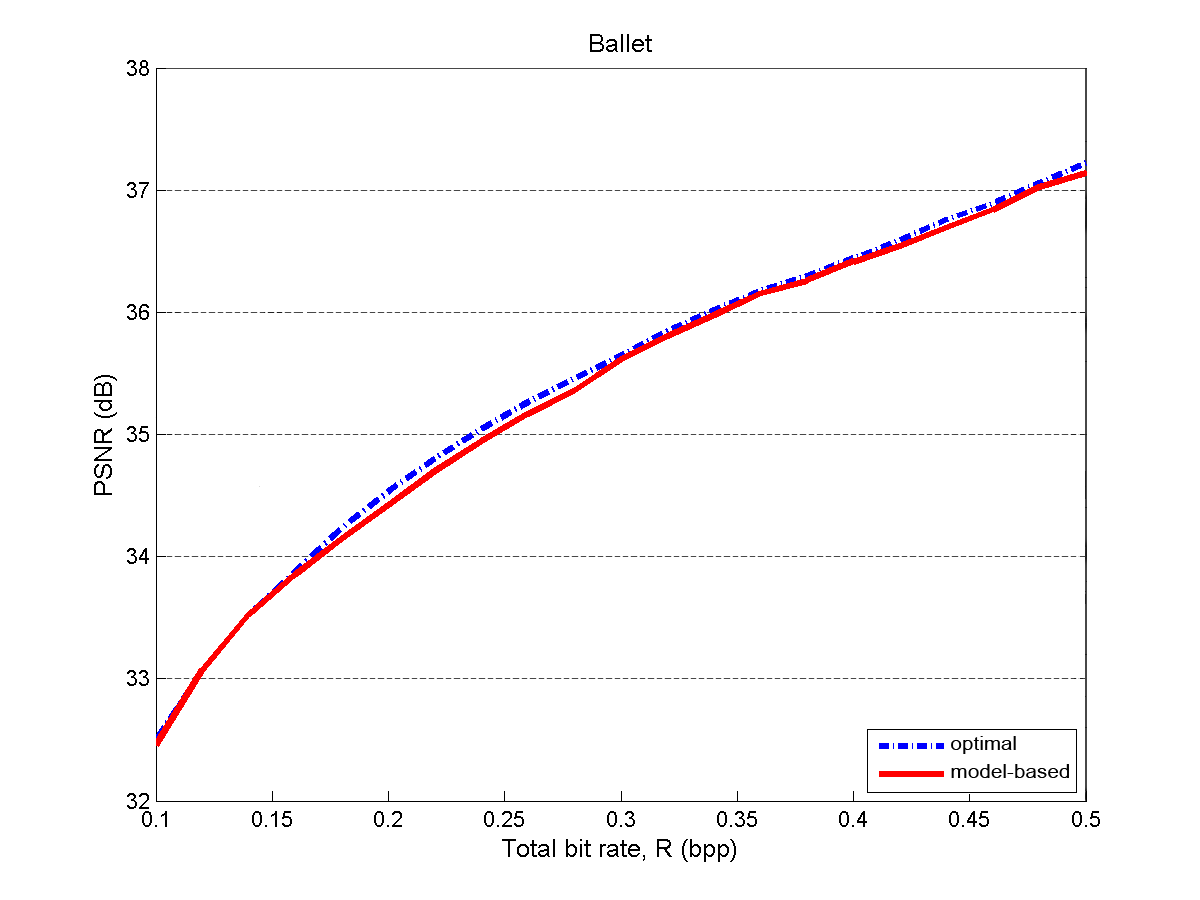,width=7cm}}
\end{minipage}
\hfill
\begin{minipage}[b]{.49\linewidth}
\centering
\leftline{\epsfig{figure=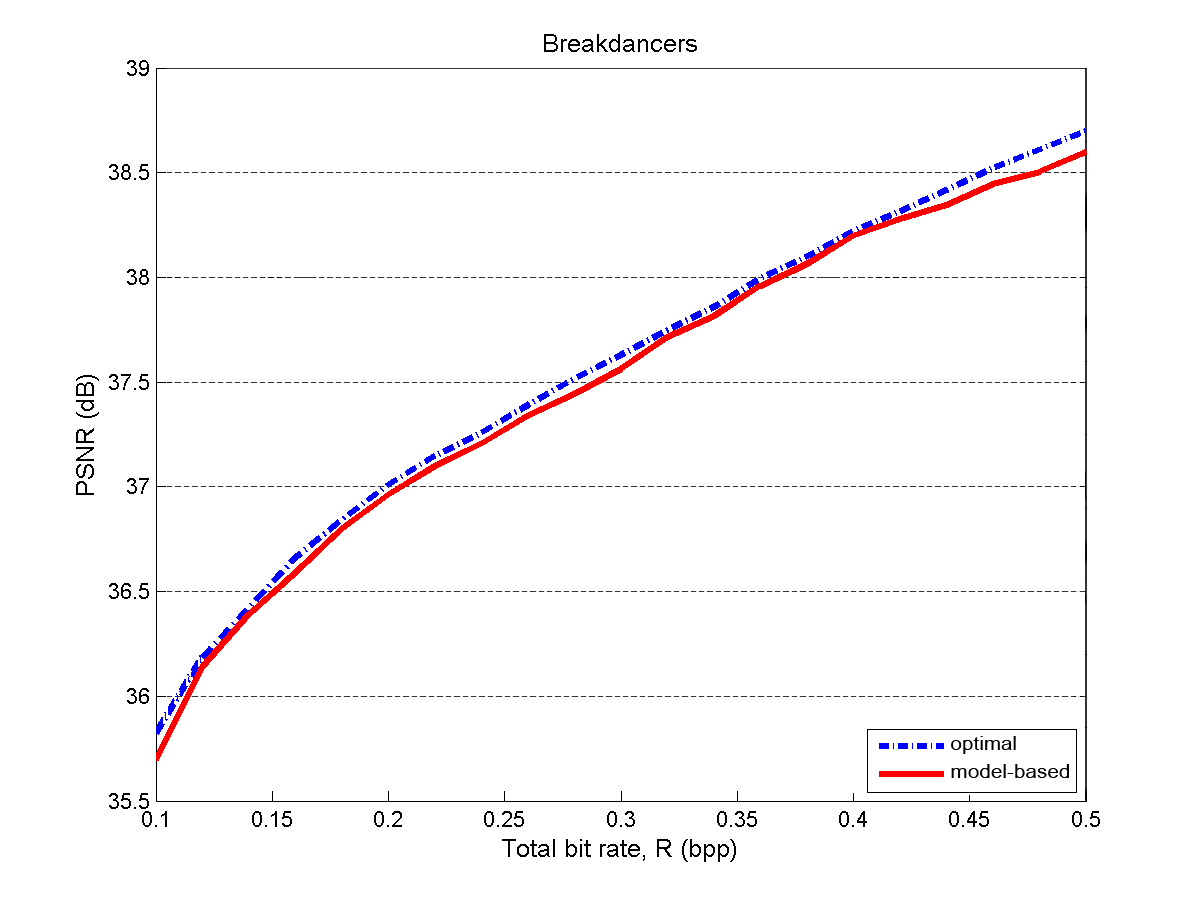,width=7cm}}
\end{minipage}
\caption{Comparison of coding performance for $\mathcal{B}^1_1(\mathbf{\mathcal{P}})$ using the proposed allocation method and the best allocation in terms of PSNR at rates ranging from 0.1 to 0.5 bpp; \emph{Ballet} (left) and \emph{Breakdancers} (right).}
\label{fig:PSNR-OneRef-OneVir}
\end{figure*}

\begin{table*}[position specifier]
\caption{Rate Allocation Results for $\mathcal{B}^1_1(\mathbf{\mathcal{P}})$ - Comparison between allocation with the proposed model and the optimal allocation, in terms of $R_t$ percentage of the total rate.}
\centering
\begin{tabular}{|cc|c|c|c|c|}
\hline
\multicolumn{2}{|c|}{Total bitrate} & 0.2 bpp & 0.3 bpp & 0.4 bpp & 0.5 bpp \\ \cline{1-6}
\multicolumn{1}{|c|}{\multirow{2}{*}{\emph{Ballet}}} &
\multicolumn{1}{|c|}{optimal} & 67.83\% & 57.78\% & 53.33\% & 37.33\%      \\ \cline{2-6}
\multicolumn{1}{|c|}{}                        &
\multicolumn{1}{|c|}{model-based} & 54.61\% & 49.61\% & 48.46\% & 48.36\%      \\ \hline\hline
\multicolumn{1}{|c|}{\multirow{2}{*}{\emph{Breakdancers}}} &
\multicolumn{1}{|c|}{optimal} & 80.91\% & 75.56\% & 70.32\% & 73.33\% \\ \cline{2-6}
\multicolumn{1}{|c|}{}                        &
\multicolumn{1}{|c|}{model-based} & 75.72\% & 74.54\% & 75.09\% & 75.78\% \\ \cline{1-6}
\end{tabular}
\label{tab:RT-percentage-Oneref-Onevir}
\end{table*}

Table \ref{tab:RT-percentage-Oneref-Onevir} shows the percentage of the total rate that is used for coding texture for different target bit rates. Clearly our model-based allocation follows closely the best allocation. Figure \ref{fig:RT-Percent-OneRef-OneVir} further shows the best and model-based allocations versus bit rate in terms of $R_t$ percentage. Additionally, two dotted curves are presented which are the higher and lower bounds on $R_t$ allocation where the PSNR loss compared to the best allocation remains below 0.2 dB.

We study now the performance of a priori given rate allocations, which are commonly adopted in practice. We consider several such allocations, where the values of $R_t$ relative to the total budget spans a range of 20 to 80 \%. Table \ref{tab:fixed-allocation-Oneref-Onevir} shows the average PSNR loss compared to the best allocation in these cases. All these results are the average over frames 0, 49 and 99 in both datasets. We compare the performance of the rate allocation estimated with our RD model and we show that our allocation is always better. Figure \ref{fig:RT-Percent-OneRef-OneVir} further shows that using a model-based allocation instead of a priori allocation is more important at low bit rates or in images with close to camera objects (like \emph{Ballet}). Depending on the dataset, the best a priori allocation occurs at different $R_t$ percentages. In our proposed allocation, the results are close to optimal in both datasets as the model adopts to the scene content. The last two rows of Table \ref{tab:fixed-allocation-Oneref-Onevir} shows the average benefit of our model compared to a fixed rate allocation.

\begin{figure*}[tb]

\begin{minipage}[b]{.49\linewidth}
\centering
\rightline{\epsfig{figure=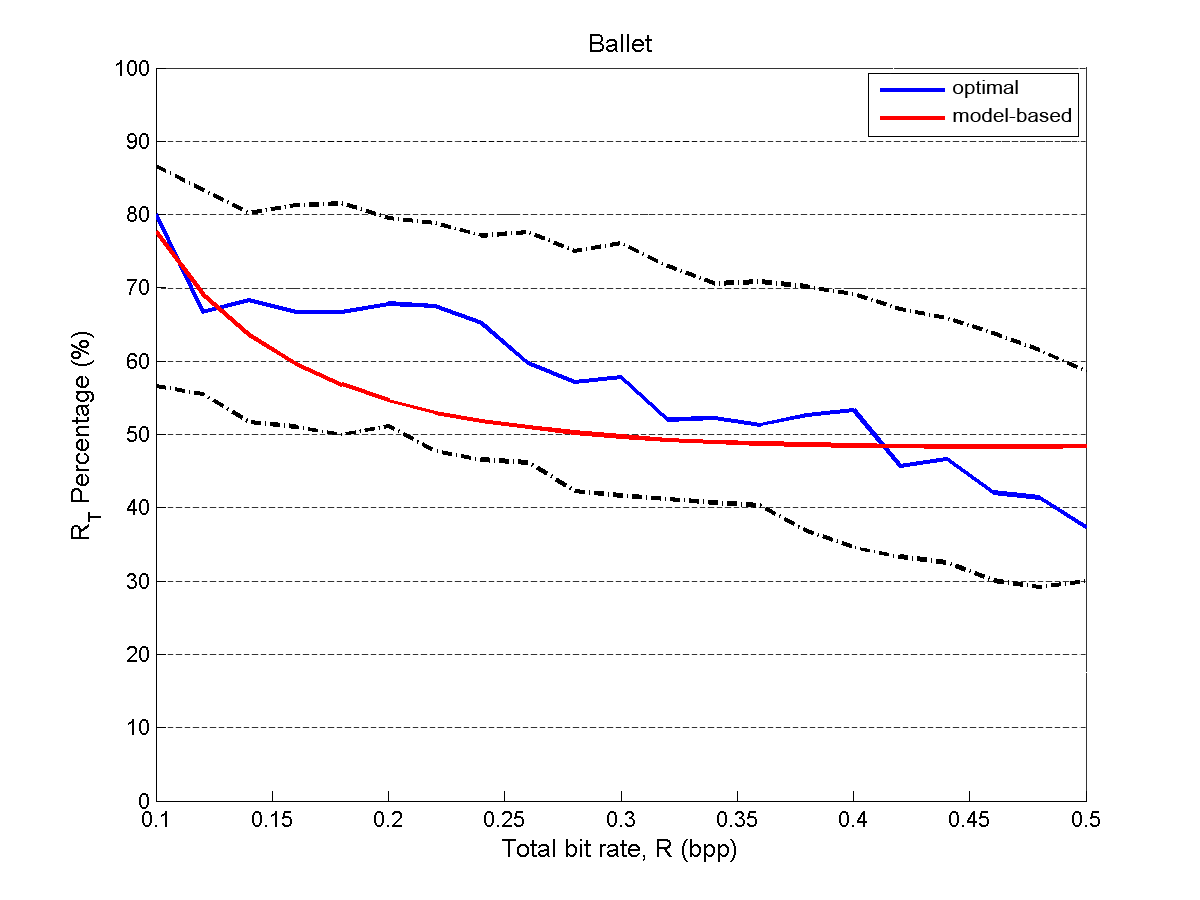,width=7cm}}
\end{minipage}
\hfill
\begin{minipage}[b]{.49\linewidth}
\centering
\leftline{\epsfig{figure=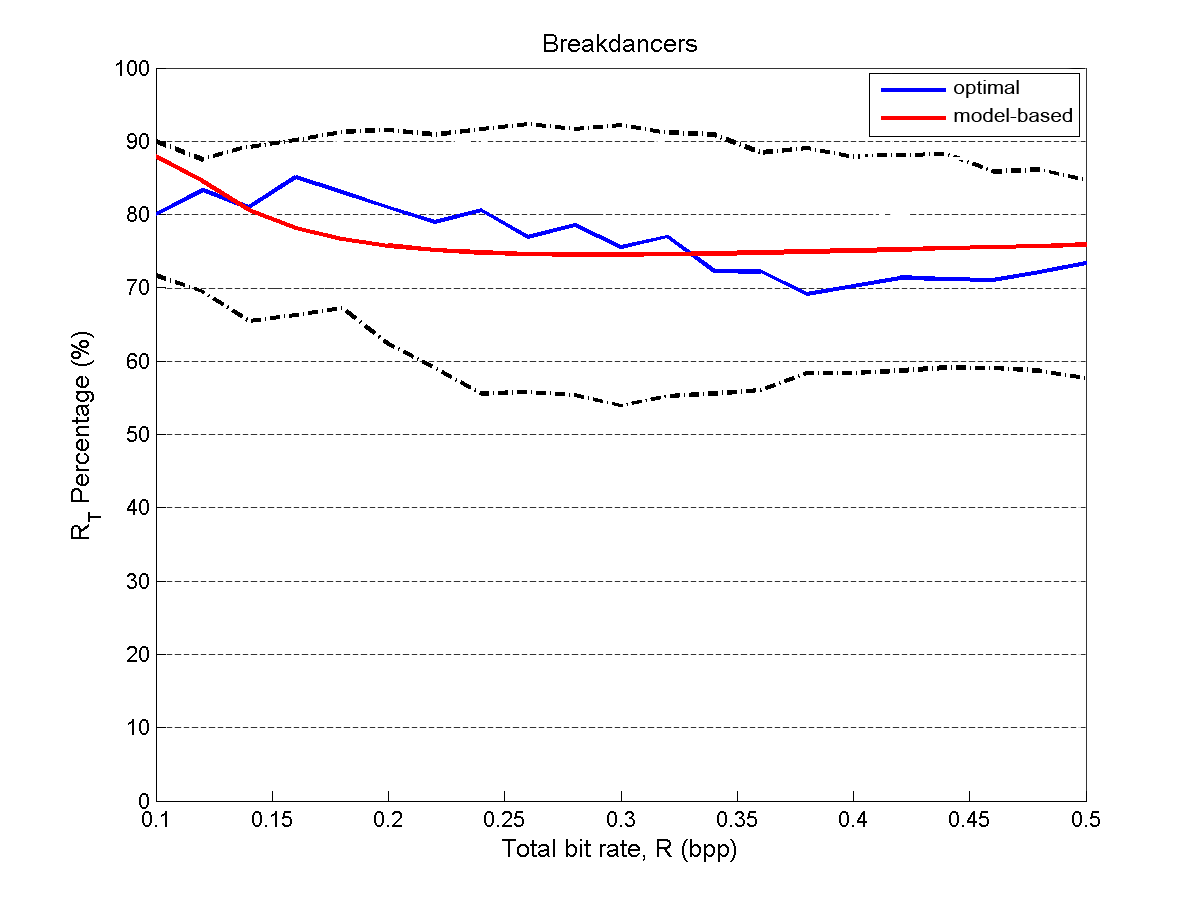,width=7cm}}
\end{minipage}
\caption{Rate allocation results of $\mathcal{B}^1_1(\mathbf{\mathcal{P}})$ using our proposed method and the optimal allocation in terms of $R_t$ percentage of total rates ranging from 0.1 to 0.5 bpp; \emph{Ballet} (left) and \emph{Breakdancers} (right). The black dashed curves show the bounds within which the difference in PSNR quality with optimal allocation remains less than or equal to 0.2 dB. }
\label{fig:RT-Percent-OneRef-OneVir}
\end{figure*}

\begin{table*}[position specifier]
\caption{Performance penalty for fixed allocation in $\mathcal{B}^1_1(\mathbf{\mathcal{P}})$ - Comparison between the proposed model and a priori allocation policies in terms of average and maximum differences to the best achievable PSNR at total rates ranging from 0.1 to 0.5. The column headers indicate the a priori allocation of $R_t$ relatively to the total rate.}
\centering
\begin{tabular}{|cc|c|c|c|c|c|c|c|c|c|}
\hline
\multicolumn{2}{|c|}{$R_t$ percentage} & 20\% & 30\% & 40\% & 50\% & 60\% & 70\% & 80\% & our model\\ \cline{1-10}
\multicolumn{1}{|c|}{\multirow{2}{*}{\emph{Ballet}}} &
\multicolumn{1}{|c|}{Average (dB)} & 1.43 & 0.66 & 0.29 & 0.11 & 0.07 & 0.14 & 0.35 & \textbf{0.05}\\ \cline{2-10}
\multicolumn{1}{|c|}{}                        &
\multicolumn{1}{|c|}{Maximum (dB)} & 2.20 & 1.15 & 0.63 & 0.31 & 0.21 & 0.39 & 0.77 & 0.11\\ \hline\hline
\multicolumn{1}{|c|}{\multirow{2}{*}{\emph{Breakdancers}}} &
\multicolumn{1}{|c|}{Average (dB)} & 1.97 & 1.14 & 0.68 & 0.40 & 0.21 & 0.10 & \textbf{0.06} & \textbf{0.06} \\ \cline{2-10}
\multicolumn{1}{|c|}{}                        &
\multicolumn{1}{|c|}{Maximum (dB)} & 3.23 & 2.08 & 1.33 & 0.77 & 0.44 & 0.16 & 0.11 & 0.13 \\ \hline\hline
\multicolumn{1}{|c|}{\multirow{2}{*}{\emph{Overall}}} &
\multicolumn{1}{|c|}{Average (dB)} & 1.70 & 0.90 & 0.49 & 0.26 & 0.14 & 0.12 & 0.21 & \textbf{0.06} \\ \cline{2-10}
\multicolumn{1}{|c|}{}                        &
\multicolumn{1}{|c|}{Maximum (dB)} & 2.72 & 1.62 & 0.98 & 0.54 & 0.33 & 0.28 & 0.49 & 0.12 \\ \cline{1-10}
\end{tabular}
\label{tab:fixed-allocation-Oneref-Onevir}
\end{table*}

Finally we study the effect of the distance of virtual views on the rate allocation. We vary the distance between reference and virtual view from 1 to 20 cm by only changing the value of the $x$ coordinate in the $T'$ translation vector of the virtual camera. We further fix the total bit rate to $R=0.24$ bpp. Figure \ref{fig:RT-Percent-Distance} shows the best rate allocation as a function of the distance of the virtual view. Again, these results are the average over frame 0, 49 and 99 of \emph{Ballet} and \emph{Breakdancers} datasets. Intuitively, for a given error in depth maps due to coding effects, rendering distortion should be smaller in closer virtual views than farther ones. It means that for rendering far views we need more accurate depth information for rendering far views. Alternatively, texture coding distortion plays a more important role in closer views. This is shown in Figure \ref{fig:RT-Percent-Distance} as the $R_t$ percentage decreases by increasing the distance of the virtual view. For \emph{Ballet} dataset we however observe an increase in $R_t$ after 12 cm. It is due to the nature of this scene and to the fact that we use only one camera for rendering virtual views. In this sample there are two foreground objects which are close to the camera, and, beyond a given distance, they move out of view boundaries and mostly background pixels remain. Clearly depth coding errors is less important for background regions that are far from the camera. We also show in Figure \ref{fig:RT-Percent-Distance} the model-based allocation using our RD equation in Eq. \eqref{equ: RDoptimization} where we only change $T'$. Therefore, the second of the distortion grows with the distance which means that increasing $R_d$ yields smaller distortion comparing to increasing $R_t$. The average PSNR penalty of our model-based allocation is 0.05 dB and 0.03 dB for \emph{Ballet} and \emph{Breakdancers}, respectively.

\begin{figure*}[tb]

\begin{minipage}[b]{.49\linewidth}
\centering
\rightline{\epsfig{figure=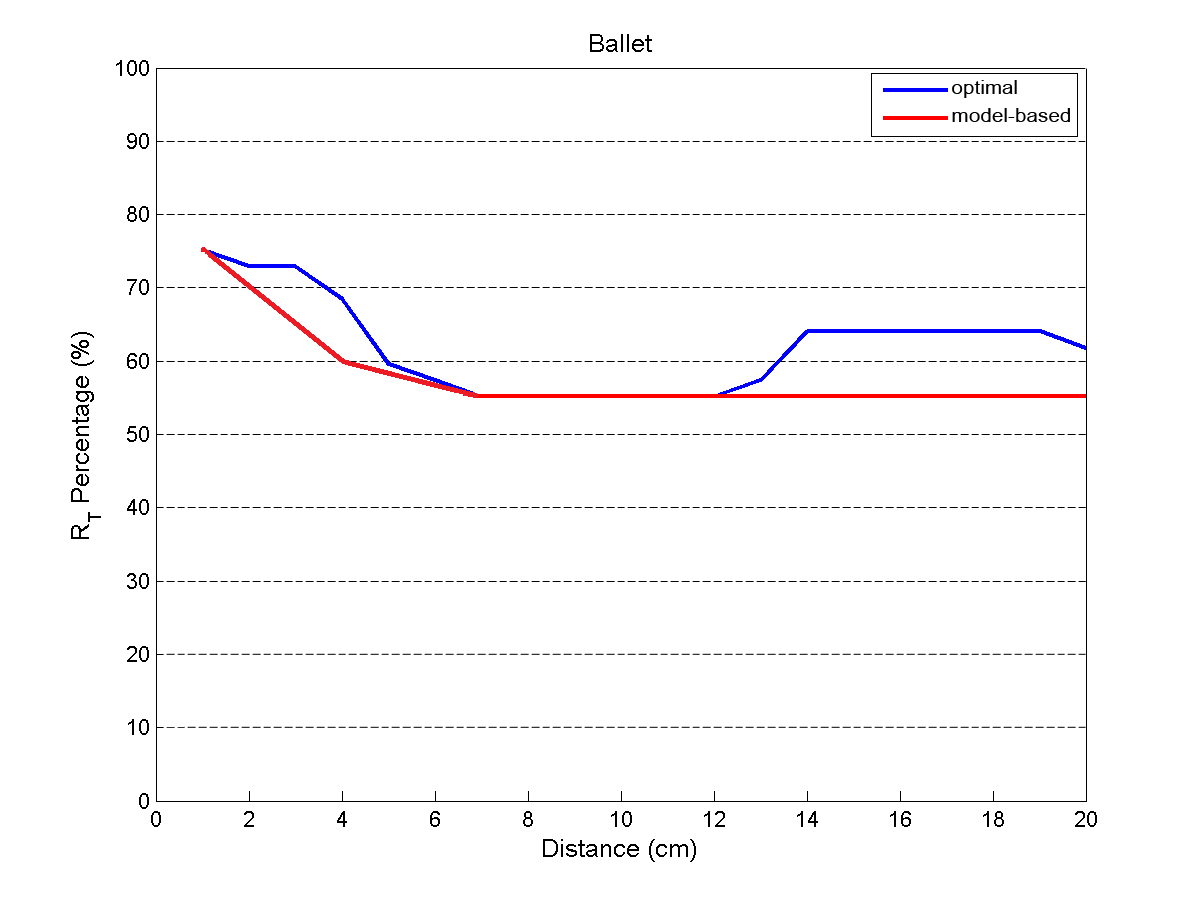,width=7cm}}
\end{minipage}
\hfill
\begin{minipage}[b]{.49\linewidth}
\centering
\leftline{\epsfig{figure=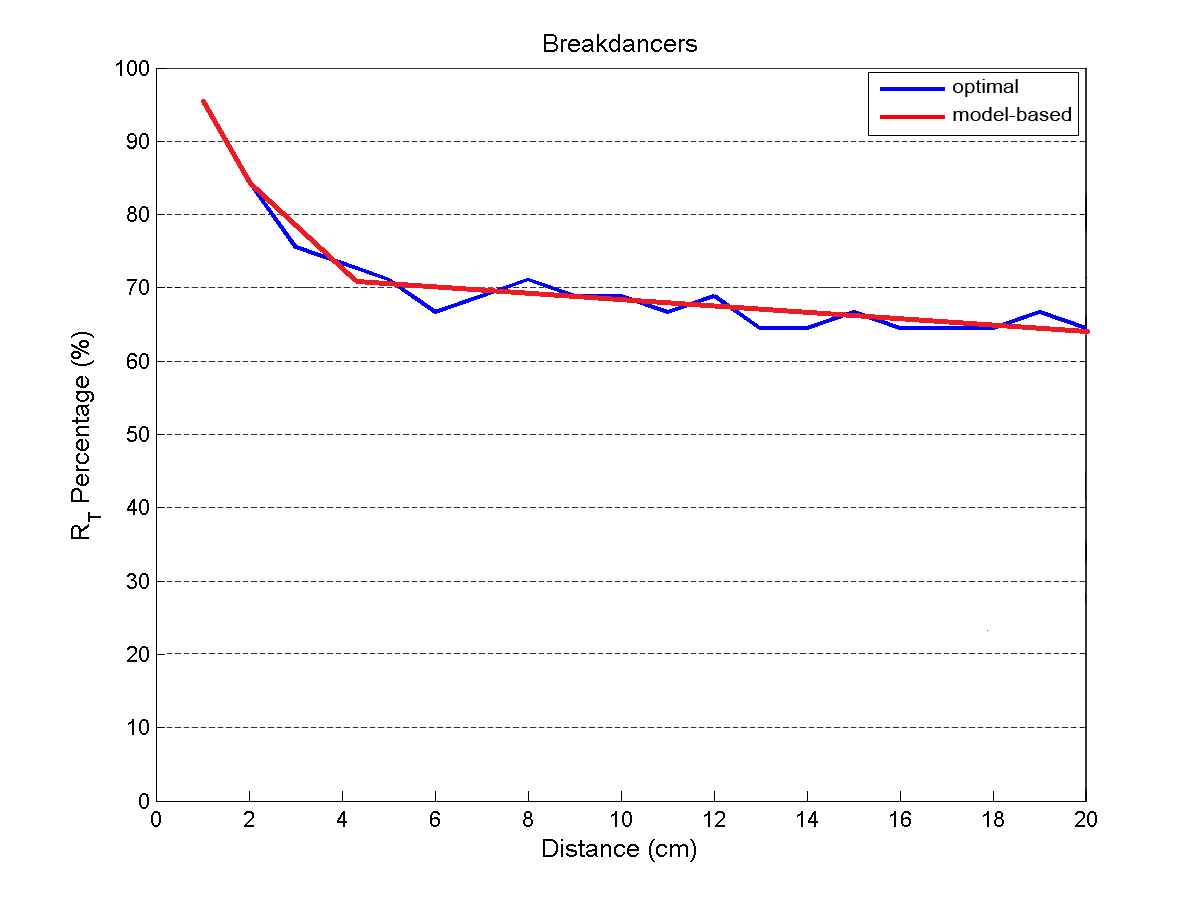,width=7cm}}
\end{minipage}
\caption{Rate allocation results of $\mathcal{B}^1_1(\mathbf{\mathcal{P}})$ using the model-based and the optimal allocation in terms of $R_t$ percentage at a total rate of 0.24 bpp; \emph{Ballet} (left) and \emph{Breakdancers} (right). The virtual view is projected at 1 to 20 centimeters from reference view.}
\label{fig:RT-Percent-Distance}
\end{figure*}

\subsection{$\mathcal{B}^1_q(\mathbf{\mathcal{P}})$ configuration}\label{ssec: OneRef-MultipleVir}
In this section we study the allocation problem for camera configuration with multiple virtual views. The camera 4 of \emph{Ballet} and \emph{Breakdancers} datasets is used as the reference camera and six virtual cameras separated by 1 cm are considered, three at each side of the reference camera. At each side the parameters of the virtual cameras are set according to camera 3 and 5, respectively.

The optimal allocation process is obtained similarly to section \ref{ssec:OneRef-OneVir}. The optimal RD surface is generated offline, for $R_t$ and $R_d$ rates between 0.02 and 0.5 bpp with 0.02 bpp steps. Then, at each bit rate $R$, the best allocation is calculated using interpolation over this RD surface. The model-based allocation is the result of solving Eq. \eqref{equ: RDoptimization} for $\mathcal{B}^1_6(\mathbf{\mathcal{P}})$. The reported distortion is the average distortion over all six virtual views and the reference view and also over the three representative frames in each set, i.e., frames 0, 49 and 99.

\begin{figure*}[tb]

\begin{minipage}[b]{.49\linewidth}
\centering
\rightline{\epsfig{figure=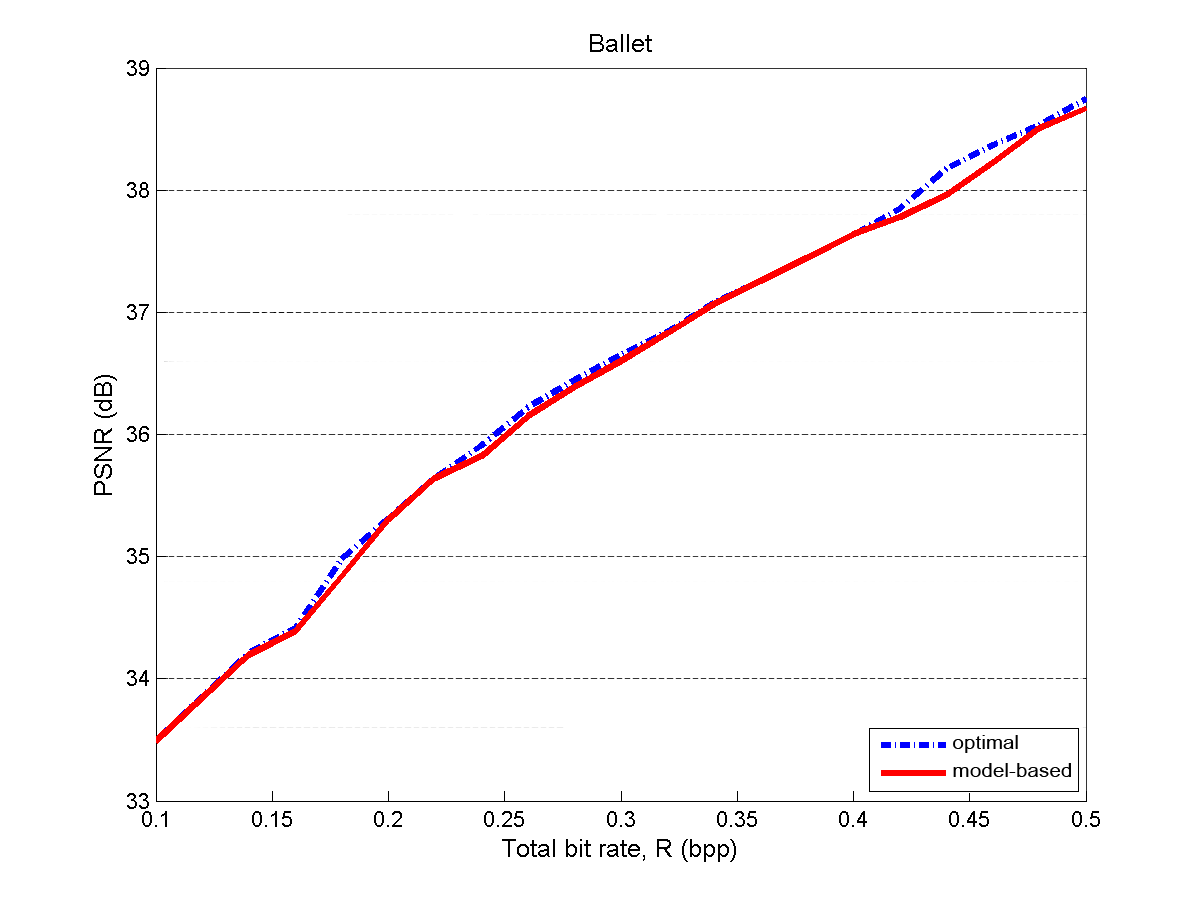,width=7cm}}
\end{minipage}
\hfill
\begin{minipage}[b]{.49\linewidth}
\centering
\leftline{\epsfig{figure=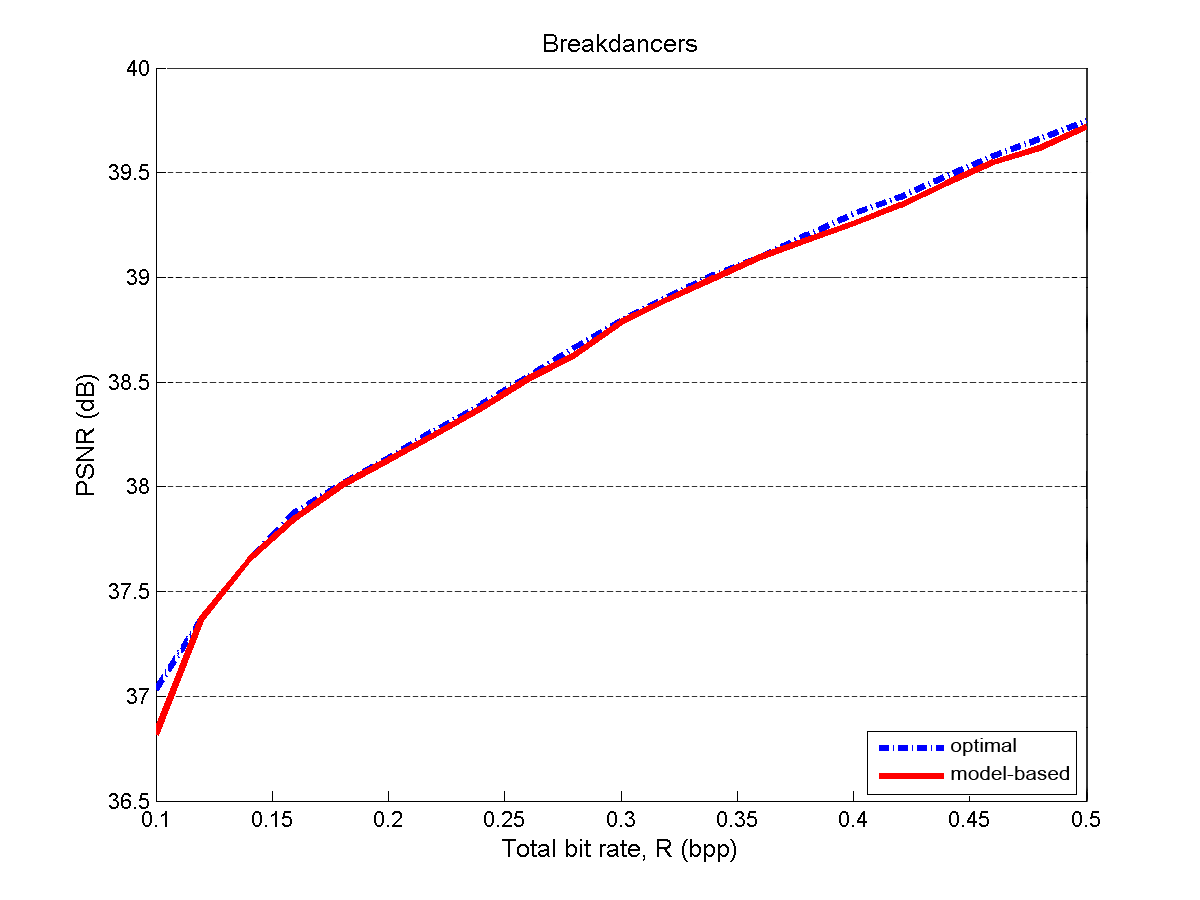,width=7cm}}
\end{minipage}
\caption{Comparison of coding performance for $\mathcal{B}^1_6(\mathbf{\mathcal{P}})$ using the model-based allocation method and the best allocation in terms of PSNR at rates ranging from 0.1 to 0.5 bpp; \emph{Ballet} (left) and \emph{Breakdancers} (right).}
\label{fig:PSNR-OneRef-6Vir}
\end{figure*}

\begin{figure*}[tb]

\begin{minipage}[b]{.49\linewidth}
\centering
\rightline{\epsfig{figure=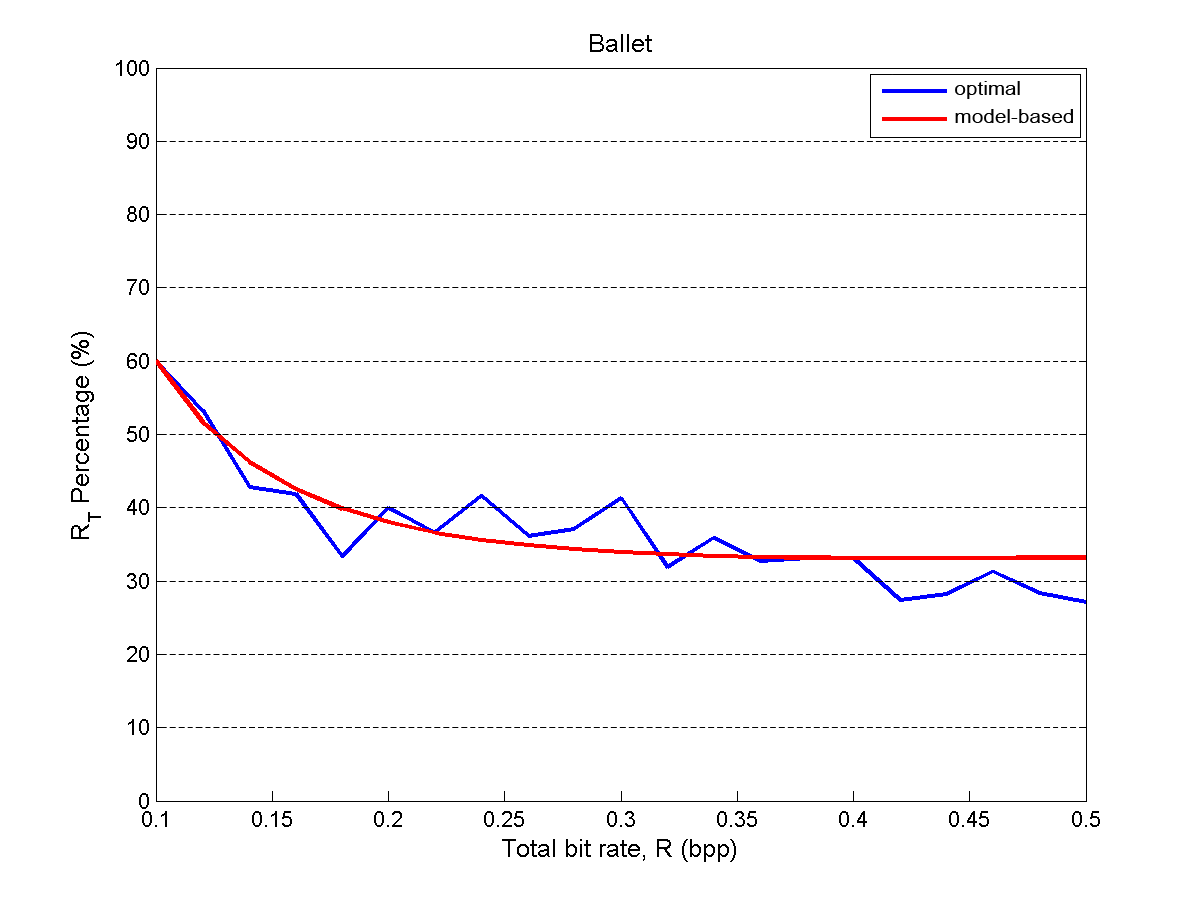,width=7cm}}
\end{minipage}
\hfill
\begin{minipage}[b]{.49\linewidth}
\centering
\leftline{\epsfig{figure=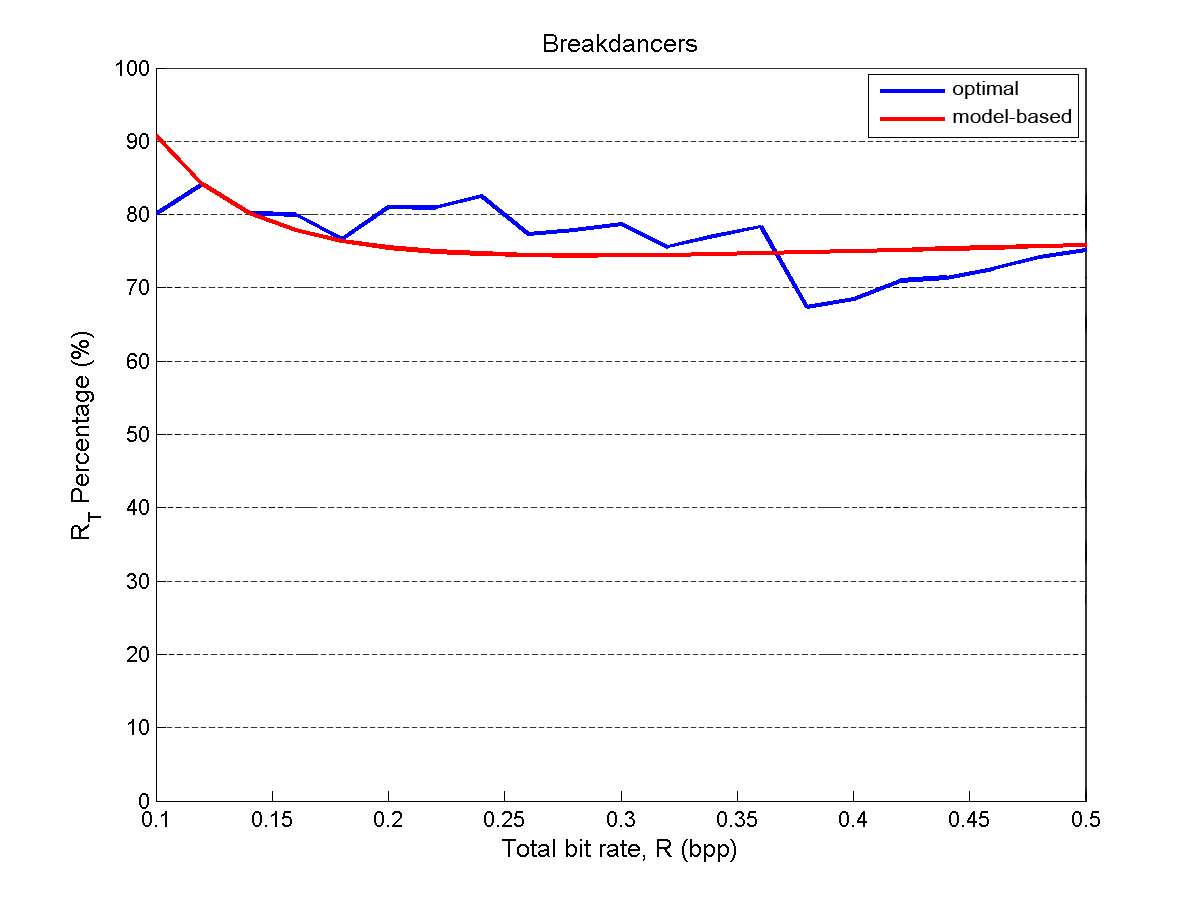,width=7cm}}
\end{minipage}
\caption{Rate allocation results of $\mathcal{B}^1_6(\mathbf{\mathcal{P}})$ using the model-based and the optimal allocation in terms of $R_t$ percentage at total rates ranging from 0.1 to 0.5 bpp; \emph{Ballet} (left) and \emph{Breakdancers} (right).}
\label{fig:RT-Percent-6Ref-OneVir}
\end{figure*}

Figure \ref{fig:PSNR-OneRef-6Vir} represents performance in terms of PSNR with respect to target bit rate, $R$, where $R$ varies between 0.1 and 0.5 bpp. The two curves correspond to the best allocation and the model-based allocation. The amount of loss due using our model is 0.05 and 0.03 dB, on average, for \emph{Ballet} and \emph{Breakdancers}, respectively. Also, the maximum difference is 0.22 and 0.21 dB, respectively. Figure \ref{fig:RT-Percent-6Ref-OneVir} shows the best and model-based allocation in terms of percentage of the total rate allocated to $R_t$, for different values of $R$. Clearly our model again performs very close to the optimal allocation. This yields to clear improvements over a priori rate allocation as given in Table \ref{tab:fixed-allocation-Oneref-Multiplevir} in case of $\mathcal{B}^1_6(\mathbf{\mathcal{P}})$.

\begin{table*}[position specifier]
\caption{Performance penalty for fixed allocation in $\mathcal{B}^1_6(\mathbf{\mathcal{P}})$ - Comparison between the proposed model and a priori allocation policies in terms of average and maximum differences to the best achievable PSNR at total rates ranging from 0.1 to 0.5. The column headers indicate the a priori allocation of $R_t$ relatively to the total rate.}
\centering
\begin{tabular}{|cc|c|c|c|c|c|c|c|c|c|}
\hline
\multicolumn{2}{|c|}{$R_t$ percentage} & 20\% & 30\% & 40\% & 50\% & 60\% & 70\% & 80\% & our model\\ \cline{1-10}
\multicolumn{1}{|c|}{\multirow{2}{*}{\emph{Ballet}}} &
\multicolumn{1}{|c|}{Average (dB)} & 0.54 & \textbf{0.11} & 0.12 & 0.32 & 0.67 & 1.21 & 1.91 & \textbf{0.05}\\ \cline{2-10}
\multicolumn{1}{|c|}{}                        &
\multicolumn{1}{|c|}{Maximum (dB)} & 1.26 & 0.47 & 0.35 & 0.70 & 1.23 & 1.90 & 3.13 & 0.22\\ \hline\hline
\multicolumn{1}{|c|}{\multirow{2}{*}{\emph{Breakdancers}}} &
\multicolumn{1}{|c|}{Average (dB)} & 2.03 & 1.15 & 0.66 & 0.36 & 0.16 & 0.05 & \textbf{0.03} & \textbf{0.03} \\ \cline{2-10}
\multicolumn{1}{|c|}{}                        &
\multicolumn{1}{|c|}{Maximum (dB)} & 3.57 & 2.28 & 1.38 & 0.78 & 0.37 & 0.12 & 0.08 & 0.21 \\ \hline\hline
\multicolumn{1}{|c|}{\multirow{2}{*}{\emph{Overall}}} &
\multicolumn{1}{|c|}{Average (dB)} & 1.29 & 0.63 & 0.39 & \textbf{0.34} & 0.42 & 0.63 & 0.97 & \textbf{0.04} \\ \cline{2-10}
\multicolumn{1}{|c|}{}                        &
\multicolumn{1}{|c|}{Maximum (dB)} & 3.57 & 2.28 & 1.38 & 0.78 & 1.23 & 1.90 & 3.13 & 0.22 \\ \cline{1-10}
\end{tabular}
\label{tab:fixed-allocation-Oneref-Multiplevir}
\end{table*}

\begin{figure*}[tb]

\begin{minipage}[b]{.49\linewidth}
\centering
\rightline{\epsfig{figure=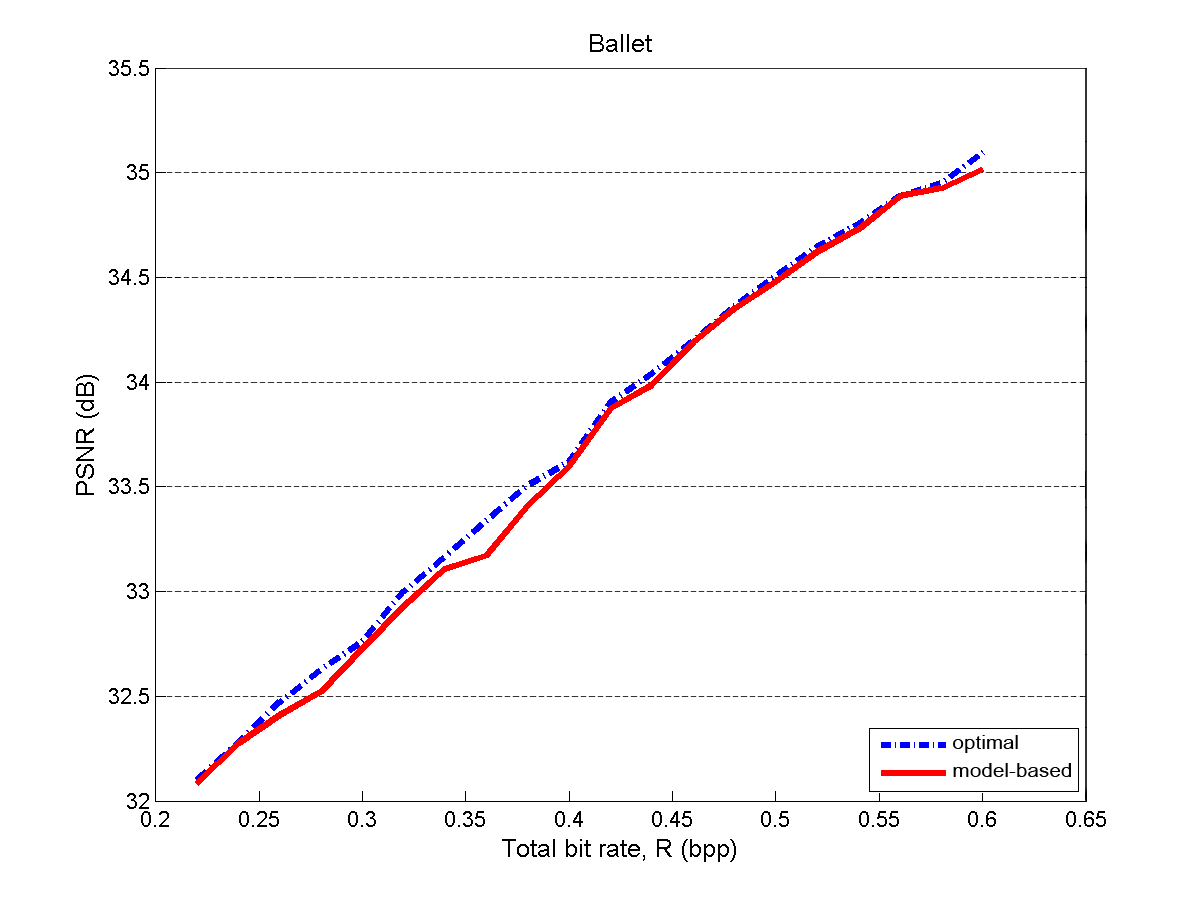,width=7cm}}
\end{minipage}
\hfill
\begin{minipage}[b]{.49\linewidth}
\centering
\leftline{\epsfig{figure=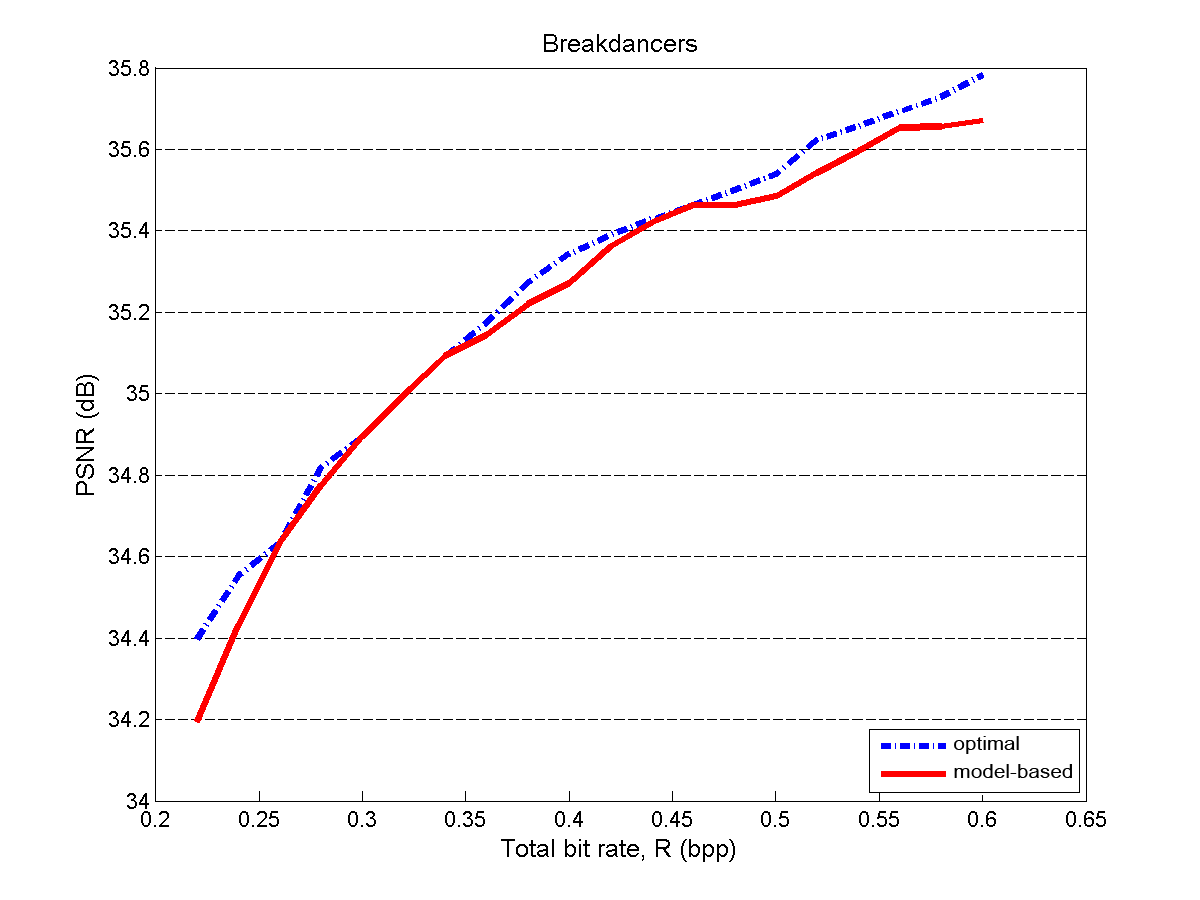,width=7cm}}
\end{minipage}
\caption{Comparison of coding performance for $\mathcal{B}^2_3(\mathbf{\mathcal{P}})$ using the model-based allocation method and the best allocation in terms of PSNR at rates ranging from 0.22 to 0.6 bpp; \emph{Ballet} (left) and \emph{Breakdancers} (right).}
\label{fig:PSNR-TwoRef-6Vir}
\end{figure*}

\begin{figure*}[tb]

\begin{minipage}[b]{.49\linewidth}
\centering
\rightline{\epsfig{figure=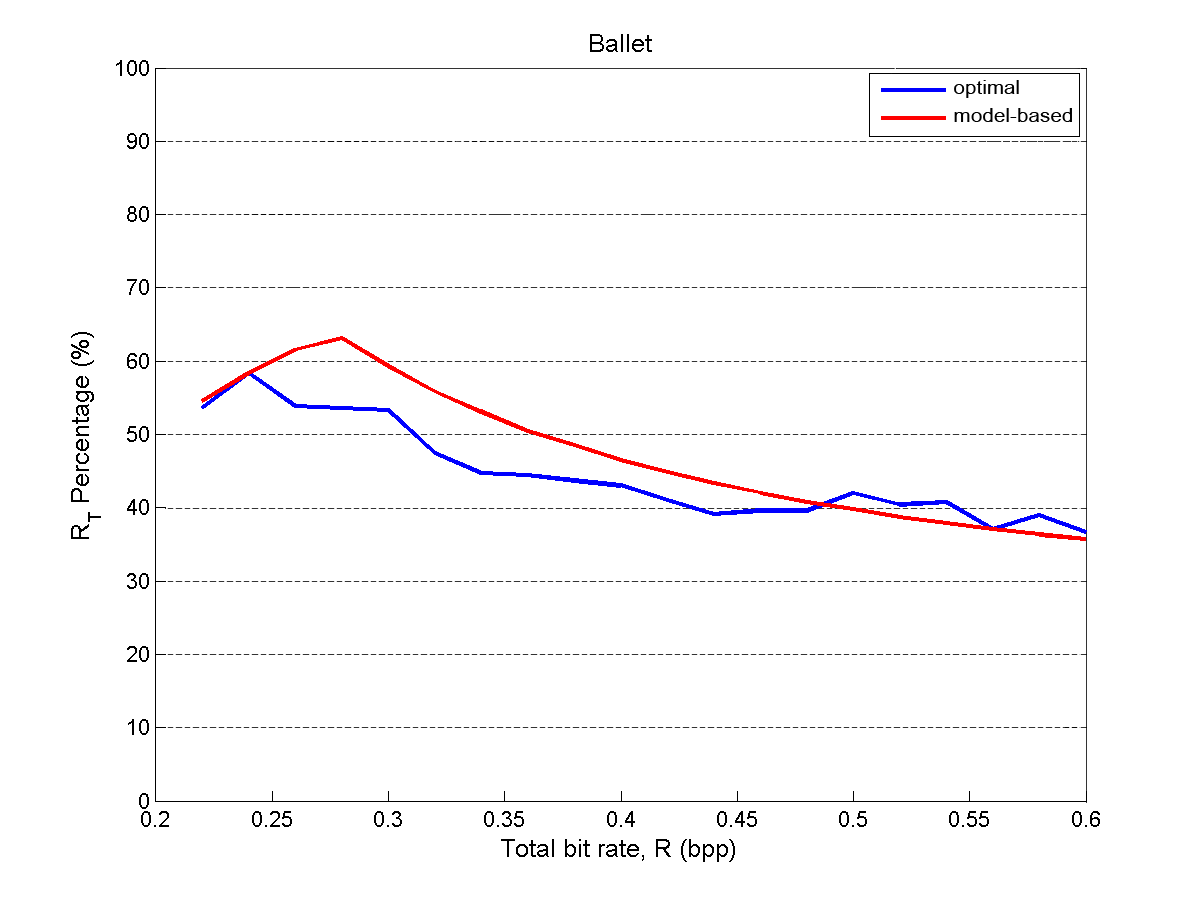,width=7cm}}
\end{minipage}
\hfill
\begin{minipage}[b]{.49\linewidth}
\centering
\leftline{\epsfig{figure=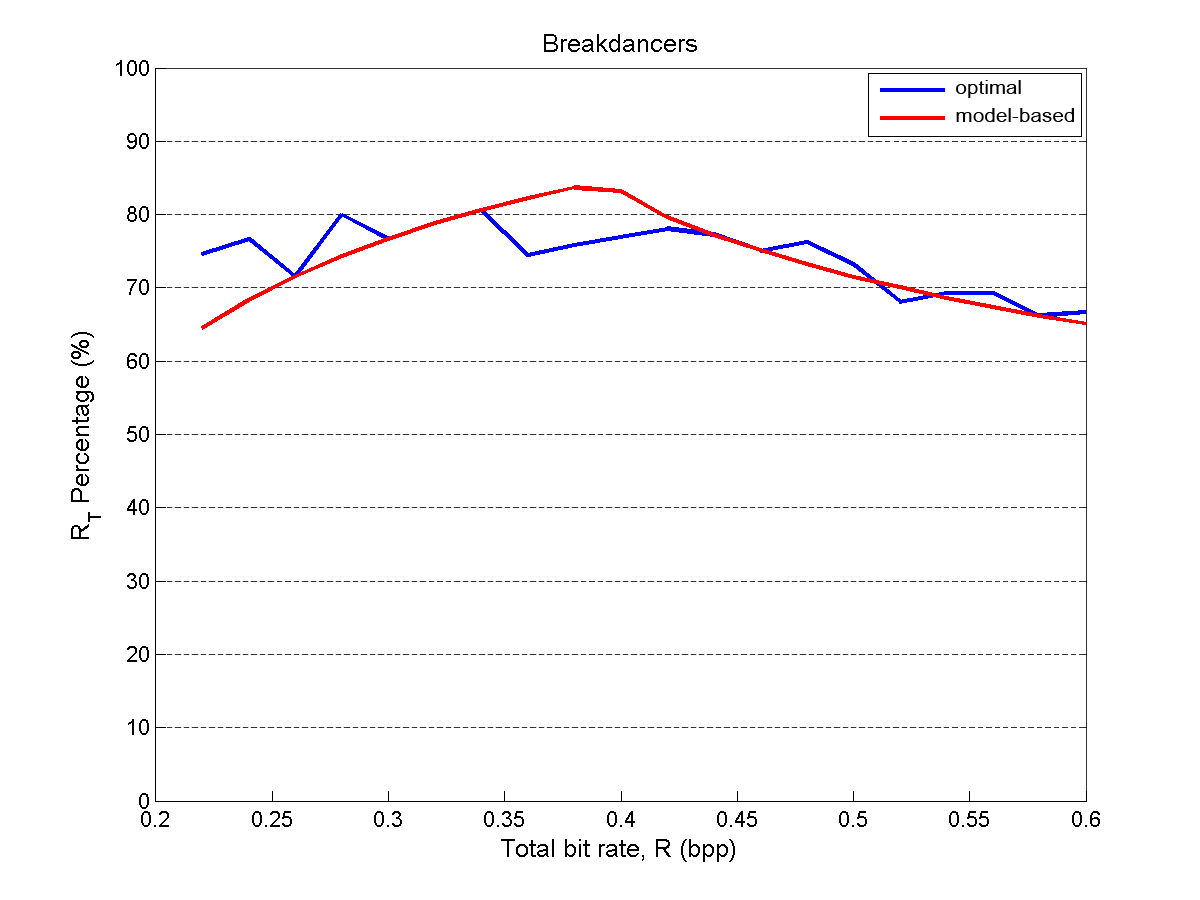,width=7cm}}
\end{minipage}
\caption{Rate allocation results of $\mathcal{B}^2_3(\mathbf{\mathcal{P}})$ using the model-based and the optimal allocation in terms of $R_t$ percentage at total rates ranging from 0.22 to 0.6 bpp; \emph{Ballet} (left) and \emph{Breakdancers} (right).}
\label{fig:RT-Percent-TwoRef-3Vir}
\end{figure*}

\begin{table*}[position specifier]
\caption{Performance penalty for fixed allocation in $\mathcal{B}^2_3(\mathbf{\mathcal{P}})$ - Comparison between the proposed model and a priori allocation policies in terms of average and maximum differences to the best achievable PSNR at total rates ranging from 0.22 to 0.6. The column headers indicate the a priori allocation of $R_t$ relatively to the total rate.}
\centering
\begin{tabular}{|cc|c|c|c|c|c|c|c|c|c|}
\hline
\multicolumn{2}{|c|}{$R_t$ percentage} & 20\% & 30\% & 40\% & 50\% & 60\% & 70\% & 80\% & our model\\ \cline{1-10}
\multicolumn{1}{|c|}{\multirow{2}{*}{\emph{Ballet}}} &
\multicolumn{1}{|c|}{Average (dB)} & 0.93 & 0.26 & \textbf{0.06} & 0.11 & 0.35 & 0.68 & 1.15 & \textbf{0.05}\\ \cline{2-10}
\multicolumn{1}{|c|}{}                        &
\multicolumn{1}{|c|}{Maximum (dB)} & 1.49 & 0.56 & 0.21 & 0.24 & 0.65 & 1.28 & 1.79 & 0.17\\ \hline\hline
\multicolumn{1}{|c|}{\multirow{2}{*}{\emph{Breakdancers}}} &
\multicolumn{1}{|c|}{Average (dB)} & 2.15 & 1.19 & 0.68 & 0.38 & 0.20 & 0.11 & \textbf{0.08} & \textbf{0.05} \\ \cline{2-10}
\multicolumn{1}{|c|}{}                        &
\multicolumn{1}{|c|}{Maximum (dB)} & 3.16 & 1.95 & 1.21 & 0.66 & 0.47 & 0.33 & 0.21 & 0.20 \\ \hline\hline
\multicolumn{1}{|c|}{\multirow{2}{*}{\emph{Overall}}} &
\multicolumn{1}{|c|}{Average (dB)} & 1.54 & 0.73 & 0.37 & \textbf{0.25} & 0.28 & 0.40 & 0.62 & \textbf{0.05} \\ \cline{2-10}
\multicolumn{1}{|c|}{}                        &
\multicolumn{1}{|c|}{Maximum (dB)} & 3.16 & 1.95 & 1.21 & 0.66 & 0.65 & 1.28 & 1.79 & 0.19 \\ \cline{1-10}
\end{tabular}
\label{tab:fixed-allocation-Tworef-Multiplevir}
\end{table*}

\subsection{$\mathcal{B}^p_q(\mathbf{\mathcal{P}})$ configuration}
We now consider the most general configuration, $\mathcal{B}^p_q(\mathbf{\mathcal{P}})$, with two reference cameras ($p=2$) and three equally spaced virtual views between them ($q=3$). The cameras 4 and 5 are considered as the two reference views and $A'_j$ and $R'_j$, $j=1,2,3$, for virtual views are set as the average of intrinsic and rotation matrices of our reference cameras. Each virtual view $v_j$ is generated in two steps. If $\pi$ is the position of $v_j$, then each of the reference views are projected into $\pi$ using depth map information. This step produces $v_{j,r}$ and $v_{j,l}$ as projection results from the right and left cameras, respectively. Next, we have
\begin{equation}\label{equ: v_j}
v_j=\frac{d_{j,l}}{d}v_{j,r}+\frac{d_{j,r}}{d}v_{j,l}
\end{equation}
where $d$ is the distance between two reference cameras, while $d_{j,l}$ and $d_{j,r}$ are the distances between $v_j$ and the left and right reference cameras, respectively.

The allocation problem in this case consists of distributing the available bit budget between two reference views and two depth maps. For comparison purposes, we calculate a DR hypersurface of the best allocation with $R_{t_1}$, $R_{t_2}$, $R_{d_1}$ and $R_{d_2}$ ranging from 0.1 to 0.6 bpp with 0.05 steps. Then for each target bit rate, $R$, the best allocation is the minimum of the resulting curve from cutting this hypersurface with the hyperplane $R_{t_1}+R_{t_2}+R_{d_1}+R_{d_2}=R$.

Figure \ref{fig:PSNR-TwoRef-6Vir} compares the best allocation and the model-based allocation in Eq. \eqref{equ: RDoptimization} for \emph{Ballet} and \emph{Breakdancers} datasets and target bit rates ranging from 0.2 to 0.6 bpp. Our allocation model yields to 0.05 dB loss in average in both cases and a maximum loss of 0.17 and 0.20 dB for \emph{Ballet} and \emph{Breakdancers}, respectively. Figure \ref{fig:RT-Percent-TwoRef-3Vir} shows the best and estimated allocations in terms of the percentage of the texture bits $(R_{t_1}+R_{t_2})$ relatively to the total bit rate. The advantage of using our model over the commonly used strategy of a priori rate allocation is shown in Table \ref{tab:fixed-allocation-Tworef-Multiplevir}. In the a priori allocation the bit rate assigned to each reference view and depth map is equal. For instance, in $\mathcal{B}^2_3(\mathbf{\mathcal{P}})$, if the total bit rate is 0.4 bpp and the a priori allocation is $40\%$, $R_{t_1}=R_{t_2}=0.08$ and $R_{d_1}=R_{d_2}=0.11$ bpp. Clearly our model outperforms the a priori allocation due to adaptivity to content and setup. From Tables \ref{tab:fixed-allocation-Oneref-Onevir} to \ref{tab:fixed-allocation-Tworef-Multiplevir}, we can conclude that the best performance of an a priori allocation strategy depends on the number of reference and virtual views and on the scene content. While our model-based allocation works well in all cases and gives this opportunity to determine number of virtual views later at decoder side.

\section{Conclusion}
We have addressed the rate-distortion analysis of multiview coding in a depth-image-based rendering context. In particular, we have shown that the distortion in the reconstruction of camera and virtual views at decoder is driven by the coding artifacts in both the reference images and the depth information. We have proposed a simple yet accurate model of the rate-distortion characteristics for simple scenes and different camera configurations. We have used our novel model for deriving effective allocation of bit rate between reference and depth images. One of the interesting features of our algorithm, beyond its simplicity, consists in avoiding the need for view synthesis at encoder, contrarily to what is generally used in state-of-the-art solutions. We finally demonstrate in extensive experiments that our simple model nicely extends to complex multiview scenes with arbitrary numbers of reference and virtual views. It leads to an effective allocation of bit rate with close-to-optimal quality under various rate constraints. In particular, our rate allocation outperforms common strategies based on static rate allocation, since it is adaptive to the scene content. Finally, we plan to extend our analysis to multiview video encoding where motion compensation poses non-trivial challenges in rate allocation algorithms due to additional coding dependencies.

\section*{Acknowledgements}
This work has been partially supported by Iran Ministry of Science, Research and Technology and the Swiss National Science Foundation under grant 200021\_126894.

\bibliographystyle{unsrt}
\bibliography{Journal_draft}

\end{document}